\PassOptionsToPackage{dvipsnames,table}{xcolor}
\documentclass[letterpaper,10pt]{vistapreprint}
\usepackage{amsmath,amssymb,amsthm,array,tabularx}
\usepackage{enumitem,needspace,float,xurl,titletoc,fontawesome5}
\setcitestyle{authoryear,round,citesep={;},aysep={,},yysep={;}}

\definecolor{vistaMailDai}{HTML}{9B303F}
\definecolor{vistaMailSong}{HTML}{2959A7}
\newcommand{\VISTAEmailIcon}[1]{\textcolor{#1}{\faIcon[regular]{envelope}}}
\definecolor{vistaNavy}{HTML}{0C3791}
\definecolor{vistaIce}{HTML}{F5FDFF}
\definecolor{vistaPink}{HTML}{FF9AC7}
\definecolor{vistaBlue}{HTML}{3BB2FB}
\definecolor{vistaYellow}{HTML}{FECD14}
\definecolor{vistaMint}{HTML}{5FC8BB}
\definecolor{vistaCoral}{HTML}{FC4748}
\definecolor{vistaPurple}{HTML}{A55DCD}
\definecolor{vistaStageOne}{HTML}{FEB2EA}
\definecolor{vistaStageTwo}{HTML}{FF9AC7}
\definecolor{vistaStageThree}{HTML}{F96144}
\definecolor{vistaStageFour}{HTML}{CC2927}
\colorlet{vistaReference}{vistaPink} \colorlet{vistaCongruence}{vistaYellow} \colorlet{vistaExpectation}{vistaBlue} \colorlet{vistaAgency}{vistaPurple} \colorlet{vistaCoping}{vistaMint} \colorlet{vistaRegulation}{vistaCoral} \colorlet{vistaIntervention}{vistaCoral}
\colorlet{vistaEvidence}{vistaBlue}
\colorlet{vistaAppraisal}{vistaPink}
\colorlet{vistaArbitration}{vistaPurple}
\colorlet{vistaPrediction}{vistaYellow}
\colorlet{vistaTeacher}{vistaMint}
\colorlet{vistaink}{black}
\colorlet{vistatext}{black}
\colorlet{vistablue}{vistaNavy}
\colorlet{vistasky}{vistaBlue!12!white}
\colorlet{vistalinen}{vistaPink!16!white}
\colorlet{vistared}{black}
\newif\ifVISTATableBody
\AtBeginEnvironment{table}{\VISTATableBodytrue\hypersetup{citecolor=black,linkcolor=black,urlcolor=black}}
\AtBeginEnvironment{table*}{\VISTATableBodytrue\hypersetup{citecolor=black,linkcolor=black,urlcolor=black}}
\AtBeginEnvironment{tabular}{\ifVISTATableBody\color{vistatext}\fi}
\AtBeginEnvironment{tabularx}{\ifVISTATableBody\color{vistatext}\fi}
\newcommand{\VISTAHeaderRow}{\rowcolor{vistasky}}
\newcommand{\VISTAFocusRow}{\rowcolor{vistalinen}}
\newcommand{\VISTAHeader}[1]{\textcolor{vistatext}{\textbf{#1}}}
\newcommand{\VISTAFocus}[1]{\textcolor{vistared}{#1}}
\makeatletter
\newcommand{\VISTAEdgeCell}[3]{\gdef\CT@cell@color{\CT@color{#1}\@tempdimb#2\relax\@tempdimc#3\relax
  \global\let\CT@cell@color\relax}}
\makeatother
\newcommand{\VISTAFirstHeader}[1]{\VISTAEdgeCell{vistasky}{0pt}{\tabcolsep}\VISTAHeader{#1}}
\newcommand{\VISTALastHeader}[1]{\VISTAEdgeCell{vistasky}{\tabcolsep}{0pt}\VISTAHeader{#1}}
\newcommand{\VISTAFirstFocus}[1]{\VISTAEdgeCell{vistalinen}{0pt}{\tabcolsep}\VISTAFocus{#1}}
\newcommand{\VISTALastFocus}{\VISTAEdgeCell{vistalinen}{\tabcolsep}{0pt}}
\arrayrulecolor{vistablue}
\hypersetup{colorlinks=true,citecolor=metablue,linkcolor=metablue,urlcolor=metablue}
\newcommand{\method}{\textsc{VISTA}}

\newcommand{\conf}{\mathrm{conf}}
\newcommand{\cons}{\mathrm{cons}}
\newcommand{\acc}{\mathrm{Acc}}

\newcolumntype{Y}{>{\raggedright\arraybackslash}X}
\colorlet{vistateal}{vistablue}
\colorlet{vistalight}{vistasky}
\newtheorem{proposition}{Proposition}
\title{VISTA: Value-Informed Event Appraisal\\for Multimodal Emotion Conflict}
\author[1,*,\VISTAEmailIcon{vistaMailDai}]{Jiale Dai}
\author[2,*]{Liuxian Ma}
\author[1]{Xiaoke Niu}
\author[3]{Wenjing Zhang}
\authorrowbreak
\author[3]{Huiying Zhao}
\author[3]{Zhaoxiang Liu}
\author[3]{Shiguo Lian}
\author[1,\ensuremath{\dagger},\VISTAEmailIcon{vistaMailSong}]{Guojie Song}
\affiliation[1]{State Key Laboratory of General Artificial Intelligence, School of Intelligence Science and Technology, Peking University}
\affiliation[2]{College of Artificial Intelligence, Tsinghua University}
\affiliation[3]{China United Network Communications Group Co., Ltd.}
\contribution[*]{Strictly equal contribution; either author order is equally valid.}
\contribution[\ensuremath{\dagger}]{Corresponding author.}
\metadata[Keywords]{Multimodal emotion recognition, cross-modal conflict, value-informed appraisal}
\metadata[Contact]{\VISTAEmailIcon{vistaMailDai}\,\href{mailto:daijiale26@stu.pku.edu.cn}{\textcolor{black}{\texttt{daijiale26@stu.pku.edu.cn}}}\quad\VISTAEmailIcon{vistaMailSong}\,\href{mailto:gjsong@pku.edu.cn}{\textcolor{black}{\texttt{gjsong@pku.edu.cn}}}}
\hypersetup{pdftitle={VISTA: Value-Informed Event Appraisal for Multimodal Emotion Conflict},pdfauthor={Jiale Dai, Liuxian Ma, Xiaoke Niu, Wenjing Zhang, Huiying Zhao, Zhaoxiang Liu, Shiguo Lian, Guojie Song},pdfsubject={Value-informed appraisal for multimodal emotion recognition},pdfkeywords={multimodal emotion recognition, cross-modal conflict, value-informed appraisal, modality arbitration}}
\renewcommand{\arraystretch}{1.08}
\newcommand{\VISTARegisterContentsTitle}[2]{\expandafter\def\csname VISTA@toc@#1\endcsname{#2}}
\newcommand{\VISTAContentsTitle}[1]{\ifcsname VISTA@toc@#1\endcsname
    \csname VISTA@toc@#1\endcsname
  \else#1\fi}
\VISTARegisterContentsTitle{Reading the seven appraisal fields}{The seven appraisal fields}
\VISTARegisterContentsTitle{The unit of appraisal and its evidence}{Appraisal unit and evidence}
\VISTARegisterContentsTitle{Evidence anchors and inferential distinctions}{Evidence anchors}
\VISTARegisterContentsTitle{Underspecified scenes and unresolved appraisals}{Underspecified scenes}
\VISTARegisterContentsTitle{Appraisal interfaces and their decision roles}{Related appraisal interfaces}
\VISTARegisterContentsTitle{Paired scenes for concern-relative interpretation}{Paired scene interpretations}
\VISTARegisterContentsTitle{The same outcome can support different concerns}{Outcome and concern}
\VISTARegisterContentsTitle{The same setback can be expressed under different social demands}{Public and private expression}
\VISTARegisterContentsTitle{Agency and norms distinguish disappointment from blame}{Agency, norms, and blame}
\VISTARegisterContentsTitle{Cue diagnosticity and representational interpretation}{Cue meaning and representation}
\VISTARegisterContentsTitle{Emotion expectation and cue diagnosticity}{Emotion priors and cue meaning}
\VISTARegisterContentsTitle{Uncertain and imperfect appraisal}{Appraisal uncertainty}
\VISTARegisterContentsTitle{What a predicted representation can contribute}{Representation refinement}
\VISTARegisterContentsTitle{A guide to appraisal and scene interventions}{Appraisal and scene interventions}
\VISTARegisterContentsTitle{Correspondence, connection, and content}{Correspondence, access, content}
\VISTARegisterContentsTitle{Scene edits and representation replacement}{Scene and representation edits}
\VISTARegisterContentsTitle{Relevant response and label change are separate judgments}{Response versus label change}
\VISTARegisterContentsTitle{How appraisal changes the weighted readout}{Appraisal and weighted readout}
\VISTARegisterContentsTitle{Semantic organization and evaluative references}{Semantic organization}
\VISTARegisterContentsTitle{Grounding the interface in observed context}{Grounding in observed context}
\VISTARegisterContentsTitle{Dataset roles and scoring conventions}{Datasets and scoring conventions}
\VISTARegisterContentsTitle{Public benchmark metrics and the normalized interface}{Benchmark metrics and scoring}
\VISTARegisterContentsTitle{Complementary evaluation settings}{Complementary evaluation settings}
\VISTARegisterContentsTitle{Conflict specificity and directional improvement}{Conflict specificity}
\VISTARegisterContentsTitle{Improvement in both conflict directions}{Both conflict directions}
\VISTARegisterContentsTitle{Complete recognition results and mechanism controls}{Recognition and mechanism tests}
\VISTARegisterContentsTitle{CA-MER: all subsets and output validity}{CA-MER and output validity}
\VISTARegisterContentsTitle{EmoMM: conflict, missingness, and their intersection}{EmoMM: conflict and missingness}
\VISTARegisterContentsTitle{CH-SIMS v2.0: increasing conflict intensity}{CH-SIMS v2: conflict intensity}
\VISTARegisterContentsTitle{The complete appraisal and training ablations}{Appraisal and training ablations}
\VISTARegisterContentsTitle{Frozen-backbone prompts and explanation quality}{Frozen prompts and explanations}
\VISTARegisterContentsTitle{Counterfactual response and irrelevant rewrites}{Counterfactual response}
\VISTARegisterContentsTitle{Two paired tests with different success criteria}{Paired tests and success criteria}
\VISTARegisterContentsTitle{Intervention dimensions and diagnostic weighting}{Diagnostic modality weights}
\VISTARegisterContentsTitle{Appraisal representation and downstream prediction}{Appraisal readout and use}
\VISTARegisterContentsTitle{A shared linear probe on frozen representations}{Frozen-backbone appraisal probes}
\VISTARegisterContentsTitle{Using appraisal to predict ten emotion intensities}{Appraisal in emotion prediction}
\VISTARegisterContentsTitle{Supervision construction, field quality, and label visibility}{Appraisal supervision and quality}
\VISTARegisterContentsTitle{From candidate pool to retained training set}{Data selection and retention}
\VISTARegisterContentsTitle{Seven-field audit on 400 candidates}{Seven-field quality audit}
\VISTARegisterContentsTitle{An isolated label-visibility diagnostic}{Label-visibility diagnostic}
\VISTARegisterContentsTitle{Computation and the structure of remaining errors}{Resources and remaining errors}
\VISTARegisterContentsTitle{Training and inference resources}{Training and inference costs}
\VISTARegisterContentsTitle{Errors corrected, introduced, and shared}{Corrected, introduced, shared errors}
\VISTARegisterContentsTitle{Ordinary conversational recognition and class recall}{MELD and class recall}
\VISTARegisterContentsTitle{Shared training and task adaptation}{Training and task adaptation}
\VISTARegisterContentsTitle{What is held fixed across the core comparison}{Common training controls}
\VISTARegisterContentsTitle{Checkpoint use across evaluation tasks}{Checkpoint use and adaptation}
\VISTARegisterContentsTitle{The complete cross-task result overview}{Cross-task result overview}
\VISTARegisterContentsTitle{Reporting conventions and reproducibility map}{Reporting and evidence map}
\VISTARegisterContentsTitle{Quantities, denominators, and precision}{Units, denominators, and precision}
\VISTARegisterContentsTitle{A complete map of the final experimental evidence}{Complete evidence map}

\VISTARegisterContentsTitle{Field types and missing supervision}{Field types and missing labels}
\VISTARegisterContentsTitle{Teacher generation and field masks}{Teacher generation and masks}
\VISTARegisterContentsTitle{Training objective and missing supervision}{Training loss and valid labels}

\newcommand{\VISTAAppendixHeadings}{\let\VISTAOriginalSection\section
  \let\VISTAOriginalSubsection\subsection
  \RenewDocumentCommand{\section}{s o m}{\IfBooleanTF{##1}{\VISTAOriginalSection*{##3}}{\IfNoValueTF{##2}{\VISTAOriginalSection[\VISTAContentsTitle{##3}]{##3}}{\VISTAOriginalSection[##2]{##3}}}}\RenewDocumentCommand{\subsection}{s o m}{\IfBooleanTF{##1}{\VISTAOriginalSubsection*{##3}}{\IfNoValueTF{##2}{\VISTAOriginalSubsection[\VISTAContentsTitle{##3}]{##3}}{\VISTAOriginalSubsection[##2]{##3}}}}}

\colorlet{vistamuted}{black}
\colorlet{vistarule}{vistaBlue!30!white}
\titlecontents{section}[1.6em]
  {\addvspace{6pt}\fontsize{10}{12}\selectfont\sffamily\bfseries\color{black}}
  {\contentslabel{1.6em}}
  {}{\hspace{.5em}\hfill\contentspage}
\titlecontents{subsection}[2.9em]
  {\addvspace{.6pt}\fontsize{9.6}{11.5}\selectfont\color{black}}
  {\contentslabel{2.5em}}
  {}{\hspace{.5em}\titlerule*[4pt]{.}\contentspage}
\newcommand{\VISTAAppendixFront}{\phantomsection\label{app:contents}\pdfbookmark[0]{Appendix contents}{vista.appendix.contents}\startcontents[vistaappendix]
  \startcontents[vistaconcepts]
  \startcontents[vistaevidence]\stopcontents[vistaevidence]
  \begingroup\hypersetup{linkcolor=black}\setlength{\parskip}{0pt}
  \noindent{\color{vistablue}\rule{\linewidth}{1.2pt}}\par
  \vspace{8pt}
  \noindent{\huge\sffamily Appendix}\hfill
  \includegraphics[height=13pt]{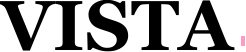}\par
  \vspace{5pt}
  {\small Value-Informed Event Appraisal for Multimodal Emotion Conflict\par}
  \vspace{7pt}
  \noindent{\color{vistarule}\rule{\linewidth}{.4pt}}\par
  \vspace{10pt}
  \noindent\begin{minipage}[t]{.48\linewidth}
    {\small\sffamily\bfseries Concepts and evaluation}\par
    \vspace{3pt}
    {\color{vistarule}\rule{\linewidth}{.4pt}}\par
    \printcontents[vistaconcepts]{}{1}{\setcounter{tocdepth}{2}}
  \end{minipage}\hfill
  \begin{minipage}[t]{.48\linewidth}
    {\small\sffamily\bfseries Results and experimental record}\par
    \vspace{3pt}
    {\color{vistarule}\rule{\linewidth}{.4pt}}\par
    \printcontents[vistaevidence]{}{1}{\setcounter{tocdepth}{2}}
  \end{minipage}\par
  \endgroup
  \clearpage
}
\abstract{
Conflicting emotional cues can be individually valid: a subdued voice may reflect a blocked goal while a smile satisfies a social obligation. Their interpretation depends on what the event means to the person. We introduce \method{} (Value-Informed Semantic Trust Arbitration), a learned seven-field appraisal interface that conditions modality arbitration on concerns, event relations, and expression conditions while retaining a joint-evidence residual. A log-odds decomposition separates emotion expectation from cue diagnosticity, motivating an interface that lets appraisal change how evidence is interpreted. With a shared Qwen2.5-Omni-7B backbone and matched training examples and steps, VISTA reaches 64.5\% conflict accuracy on CA-MER, improving on modality gating by 2.5 percentage points on conflict and 0.2 on consistency. Shuffling appraisal across scenes or removing its decision connection reduces this benefit. A common frozen-backbone probe reaches 0.600 macro CCC for appraisal readout, compared with 0.505 for emotion-only fine-tuning. Evaluations across five benchmarks connect recognition under increasing conflict with appraisal readout and downstream decision use. Together, the analyses and experiments support scene-specific appraisal as an intermediate representation that helps interpret conflicting emotional evidence.
}
\newlength{\VISTADefaultTextFloatSep}
\begin{document}
\maketitle
\section{Introduction}
\label{sec:intro}
An unsuccessful candidate smiles and says ``Congratulations'' to the winner, yet speaks in a subdued voice. The words and smile may perform public courtesy, while the voice reflects a blocked goal (Figure~\ref{fig:scene}). The recognition problem is to interpret conflicting emotional evidence about the same person.

\paragraph{A measurable source of recognition difficulty.}
Cross-modal affective disagreement affects a substantial portion of annotated data. Our audit of the public CH-SIMS labels finds different sentiment polarities across text, audio, and vision in \textbf{1,117 of 2,281 clips (49.0\%)} under a fixed three-class mapping \citep{yu2020chsims,zhang2021crossmodal}. In social-media image--text data, \citet{pan2024hybrid} report 42.5\% inconsistency on filtered MVSA-Single. These observations motivate an explicit boundary between broad affective disagreement and strong conflict.

\paragraph{Defining conflict.}
We identify pairwise conflict from modality-specific affect annotations for the same target and event. For at least two annotated modalities $\mathcal M$, annotations $q^m$, discrepancy $d$, and a fixed threshold $\tau$, define
\begin{equation}
 \kappa_\tau(X)=\mathbf 1\!\left\{\max_{m<n,\;m,n\in\mathcal M}d(q^m,q^n)>\tau\right\}.
 \label{eq:conflictdefinition}
\end{equation}
For categorical affect, $d(u,v)=\mathbf 1\{u\not\equiv v\}$ and $\tau=0$ encode disagreement, including neutral versus non-neutral polarity. Strong scalar conflict instead uses $d(u,v)=|u-v|$ and $\tau=1$ on $[-1,1]$ \citep{wang2025diffemo}. Our primary CA-MER evaluation retains its fixed partitions: exactly one unimodal label agrees with the multimodal reference; the other disagrees \citep{han2025camer}. EmoMM compares modality polarity to a joint reference \citep{sun2026emomm}. Appendix~\ref{app:conflictdefinition} gives the label audit, benchmark boundaries, and their relation to the reported conflict-strength groups.

\Needspace{5\baselineskip}
Multimodal recognition must determine both which observations to trust and what those observations imply. Cross-modal attention and shared/private representations organize complementary evidence, while balanced optimization addresses unequal use of modalities \citep{tsai2019mult,hazarika2020misa,peng2022balanced}. Conflict-aware evaluation further asks which stream agrees with the affective reference \citep{han2025camer,sun2026emomm}. Yet a clearly observed cue can support different emotions in different circumstances: a smile accompanying rejection may express courtesy, relief, or satisfaction. Assessing signal reliability alone leaves this interpretive question open. The challenge is to connect a cue to the person's relation to the event, so that its diagnostic role can change with the scene. The expressions serve different roles: congratulations address the colleague, while vocal tone can reflect the candidate's own outcome. Their disagreement can therefore inform how private concerns are managed publicly.

Appraisal theory provides that connection through goals, expectations, agency, coping, and norms \citep{lazarus1991emotion,scherer2001appraisal}. Values describe broader priorities, such as achievement or maintaining relationships \citep{schwartz2012overview}; a concern specifies what the person seeks to attain or protect in this event. Wanting the role makes rejection a setback, while public courtesy can explain the accompanying smile. The relevant reference is local to the event; a person's broader priorities need not be inferred as a fixed profile. These relations suggest an explicit interface between multimodal evidence and affect prediction: represent the event relative to the person's concerns, then use that representation when interpreting its emotional cues.

We introduce \method{}, \emph{Value-Informed Semantic Trust Arbitration}. Its seven-field appraisal organizes concerns, event relations, and expression conditions to condition both modality arbitration and affect prediction (Figure~\ref{fig:framework}). This interface makes the scene's evaluative meaning available where conflicting evidence is interpreted.

\begin{figure}[t]
\centering
\includegraphics[width=\linewidth]{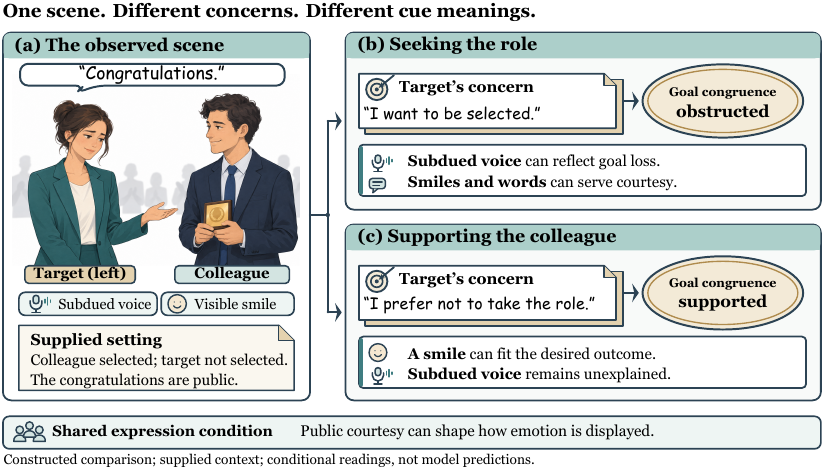}
\caption{\textbf{The same cues can have different meanings under different concerns.} This constructed comparison holds the observed event and cues fixed while changing the supplied goal of the candidate on the left. Goal congruence changes with that goal; the shared public setting supplies a condition on expression. The interpretations are conditional readings, not model predictions.}
\label{fig:scene}
\end{figure}

The contributions form a single argument. \textbf{Design rationale:} separating emotion expectation from cue diagnosticity motivates appraisal-dependent interpretation; representation refinement specifies when an explicit interface can help. \textbf{Learned interface:} VISTA connects event appraisal to arbitration and prediction. \textbf{Discriminating evidence:} matched controls locate the gain in conflict, while shuffling, generic bottlenecks, removed decision access, and paired interventions examine its source. Five benchmarks distinguish conflict resolution, appraisal readout, and recognition scope; the appendix retains complete supervision and cost results.
\section{Background and Related Work}
\label{sec:related}
\paragraph{Multimodal learning under disagreement.}
 DiffEmo operationalizes strong sentiment conflict through unimodal annotation gaps \citep{wang2025diffemo}; CA-MER and EmoMM extend evaluation to categorical conflict and polarity disagreement, with MoSEAR and CHASE as associated methods \citep{han2025camer,sun2026emomm}. CHASE uses conflict to steer attention, whereas VISTA supplies an explicit concern-relative interpretation to both arbitration and prediction. Our controlled gate comparison tests this decision role on a common backbone.

\Needspace{7\baselineskip}
\paragraph{Appraisal and value representations.}
Appraisal theory relates emotion to goals, expectations, agency, coping, and norms \citep{lazarus1991emotion,scherer2001appraisal}; regulation shapes the observable response \citep{gross1998regulation}. Early work proposed inferring appraisal from nonverbal signals \citep{mortillaro2015appraisals}. ValueNet uses value representations for emotion classification, and ValueEval studies values in arguments \citep{qiu2022valuenet,kiesel2023valueeval}. CAREBench evaluates appraisal reasoning and interventions \citep{sun2026carebench}; ECFlow uses appraisal for emotion--cause extraction \citep{liang2026ecflow}; THERADIA supplies human audiovisual appraisal annotations \citep{fournier2024theradia}. Expectation-based appraisal predicts emotion labels and shifts \citep{wang2026appraisal}; AG-CTR$^2$ instead organizes appraisal chains for support retrieval \citep{chu2026appraisal}. These tasks emphasize what appraisal predicts or retrieves. Table~\ref{tab:positioning} in Appendix~\ref{app:semantic} contrasts these roles with recognition-time cue interpretation. VISTA uses it to condition the interpretation of conflicting sensory evidence: concerns and expression conditions inform arbitration and prediction, while correspondence and connection tests examine the interface's contribution.

\paragraph{Intermediate concepts and explanations.}
Concept bottlenecks support concept supervision and intervention \citep{koh2020concept}; post-hoc variants enable concept-level model edits \citep{yuksekgonul2023posthoc}. VISTA retains a direct evidence pathway, making appraisal an auxiliary decision interface. Explanation faithfulness is a separate question \citep{turpin2023language}: dependence on internal appraisal, the content of a displayed rationale, and the semantic role of a computational route require distinct evidence.
\section{Value-Informed Semantic Trust Arbitration}
\label{sec:method}
\subsection{Scene meaning and the diagnosticity of emotional cues}
VISTA uses event meaning to interpret and combine emotional evidence. We derive how appraisal changes cue diagnosticity, implement this interaction through modality arbitration, and analyze its representational role. Let $X=(x^t,x^a,x^v,C)$ denote text, audio, video, and recognition-time context; $y$ is the affective target. A cue is \emph{diagnostic} when it distinguishes affective hypotheses within a scene.

For hypotheses $y_1,y_0$, cue $u$, and appraisal $z$, define posterior log odds $L(u,z)=\log[P(y_1\mid u,z)/P(y_0\mid u,z)]$. Bayes' rule gives, for positive probabilities,
\begin{equation}
 L(u,z)=\underbrace{\log\frac{P(y_1\mid z)}{P(y_0\mid z)}}_{b(z):\ \text{emotion expectation}}
 +\underbrace{\log\frac{P(u\mid y_1,z)}{P(u\mid y_0,z)}}_{\ell(u,z):\ \text{cue diagnosticity}},
 \label{eq:semanticlogodds}
\end{equation}
A blocked goal can change $b$; a courtesy obligation can change a smile's contribution $\ell$. At fixed $b$, the latter alone can cross the binary zero--one boundary $b+\ell=0$. For four supported cue--appraisal conditions, contrast two cues across two appraisals:
\begin{equation}
 \begin{split}
 \mathcal I&=[L(u_1,z_1)-L(u_0,z_1)]-[L(u_1,z_0)-L(u_0,z_0)]\\
 &=[\ell(u_1,z_1)-\ell(u_0,z_1)]-[\ell(u_1,z_0)-\ell(u_0,z_0)].
 \end{split}
 \label{eq:cueinteractionmain}
\end{equation}
Both $b$ terms cancel: nonzero $\mathcal I$ rules out an additive representation $L(u,z)=s(u)+t(z)$ and motivates cue--appraisal interactions. Appendix~\ref{app:proof} derives this restriction and a smile/courtesy example. Here $z$ is conceptual; predicted $\hat z=A(X)$ can itself depend on the cue, so $b(\hat z)$ need not exclude cue information. This decomposition supplies a design criterion: the decision pathway should allow appraisal to change the interpretation of a cue, beyond shifting an emotion prior.

\Needspace{6\baselineskip}
The event-level appraisal state is
\begin{equation}
 z=(g,c,e,a,k,n,r),
 \label{eq:schema}
\end{equation}
where the fields denote goal/concern, goal congruence, expectation, agency, coping/control, norm/social relevance, and expression regulation. Concern identifies what matters; congruence relates the outcome to it. Expectation, agency, and control distinguish anticipation, cause, and available remedies; norms and regulation distinguish social standards from emotional display.

Here \emph{value-informed} refers to the evaluative reference supplied by $g$ and $n$: $c$ records goal support, while $r$ describes expression conditions. Each field stores its value, confidence, evidence, and validity mask; Appendix~\ref{app:semantic} specifies the field types and missing-value rules. Section~\ref{sec:mechanism} tests removing $r$.

\begin{figure}[t]
\centering
\includegraphics[width=\linewidth]{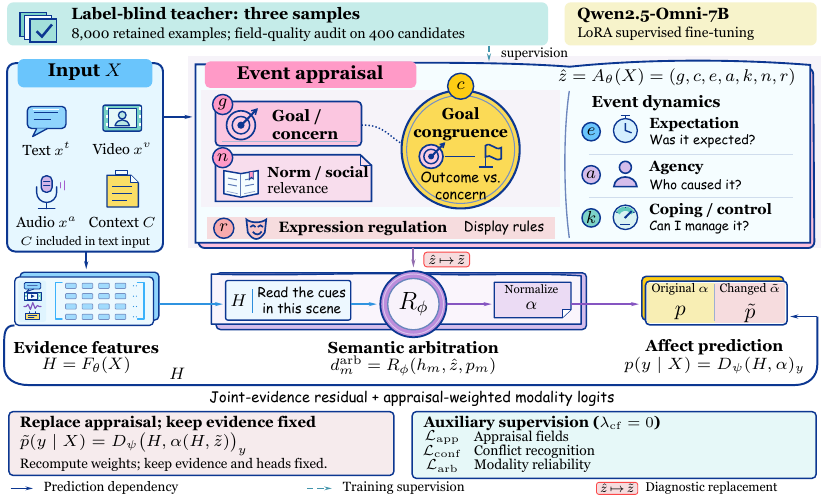}
\caption{\textbf{An explicit appraisal interface for multimodal recognition.} Appraisal sets modality weights $\alpha$; prediction combines weighted unimodal logits with joint evidence. Replacing appraisal at fixed evidence produces the readout $\tilde p$. The lower panels show this diagnostic and the supervision roles.}
\label{fig:framework}
\end{figure}

\Needspace{5\baselineskip}
\subsection{The appraisal-conditioned decision pathway}
\label{sec:pathway}
Let $H=F_{\theta}(X)=\{h_T,h_A,h_V,h_{AVT}\}$ contain separate unimodal forward passes and a joint pass. Each representation is read at the input-end position of the Thinker's final layer, after final normalization and before appraisal generation. For classification, the predicted state $\hat z=A_{\theta}(X)$ conditions the modality scores:
\begin{equation}
 \begin{gathered}
 u_m=W_mh_m,\qquad p_m=\operatorname{softmax}(u_m),\\
 d_m^{\mathrm{arb}}=R_{\phi}(h_m,\hat z,p_m),\qquad
 \alpha_m(H,\hat z)=\frac{b_m e^{d_m^{\mathrm{arb}}}}{\sum_j b_j e^{d_j^{\mathrm{arb}}}}.
 \end{gathered}
 \label{eq:arbitrationweights}
\end{equation}
where $m\in\{T,A,V\}$, $b_m\in\{0,1\}$ indicates availability, and at least one modality is available. The classifier combines logits before normalization:
\begin{equation}
 p_{\Theta}(y\mid X)=D_{\psi}(H,\alpha(H,\hat z))_y
 =\operatorname{softmax}\!\left(W_fh_{AVT}+\sum_m\alpha_m(H,\hat z)u_m\right)_y.
 \label{eq:pathway}
\end{equation}
Appraisal changes prediction through the weights, while the residual retains joint evidence. This coupling implements the cue--appraisal interaction motivated by Equation~\ref{eq:cueinteractionmain}: at fixed $H$ and availability, an appraisal change gives $\Delta\log(p_y/p_{y'})=\sum_m\Delta\alpha_m(u_{m,y}-u_{m,y'})$, with $\sum_m\Delta\alpha_m=0$. Equal branch margins therefore cancel, while greater disagreement gives scene-matched reweighting more scope to change a class preference; Appendix~\ref{app:pathidentification} proves the identity and bounds this change by the margin range. The internal $\alpha$ allocates weight among unimodal branches; the residual contributes separately. These weights differ from the predictive posterior and input-masking diagnostic $\alpha_{\mathrm{diag}}$.

\subsection{Learning and interventions}
\label{sec:learning}
The system uses Qwen2.5-Omni-7B with LoRA supervised fine-tuning \citep{qwen2025omni,hu2022lora}. Three label-blind samples from the same teacher provide seven-field pseudo-appraisals; filtering and adjudication retain 8,000 of 10,256 candidates, including partially masked field targets. A separate 400-candidate audit measures field coverage, sampling consistency, and grounding (Appendix~\ref{app:annotation}). Inference uses predicted appraisal; human appraisal is reserved for an oracle comparison.

In the isolated label-visible condition, copying is 21.0\% versus 3.0\% under label-blind supervision, while grounding is 67.0\% versus 73.0\%, despite higher field consistency. This distinction supports scene grounding as a criterion for appraisal supervision (Appendix~\ref{app:annotation}).

Training combines final-task and unimodal losses with appraisal, conflict, and arbitration supervision, using coefficients $(1,1,0.5,0.5,0.5)$ (Equation~\ref{eq:trainingobjective}). Field-type losses are normalized within fields and then over valid fields; supervision masks retain partial labels. The counterfactual coefficient is $\lambda_{\mathrm{cf}}=0$, so paired interventions evaluate learned responses to changed meaning. Appendix~\ref{app:implementation} specifies the training configuration and loss normalization.

At fixed evidence and availability, replacing $\hat z$ by $\tilde z$ gives
\begin{equation}
 p_{\Theta}(y\mid X;\tilde z)=D_{\psi}\!\left(H,\alpha(H,\tilde z)\right)_y.
 \label{eq:intervention}
\end{equation}
Shuffling tests scene correspondence; within-emotion shuffling restricts donors by emotion as an offline diagnostic. The connection ablation retains appraisal supervision and removes $z$ from arbitration, whereas rewriting an external explanation holds $H$, $z$, and $\alpha$ fixed. Section~\ref{sec:mechanism} reports these variants; Appendix~\ref{app:interventionguide} distinguishes fixed-checkpoint replacement from retraining.

\subsection{Decision-relevant representation refinement}
\label{sec:representation}
Appraisal can recover decision distinctions omitted by a compressed $H$, while remaining a function of the same input $X$.
\begin{proposition}[Decision-relevant representation refinement]
\label{prop:representationrisk}
For fixed deterministic $H=F(X)$ and $\hat Z=A(X)$ with finite labels, let $\mathcal R^*(U)=1-\mathbb E[\max_y P(Y=y\mid U)]$ denote Bayes zero--one risk. Then
\begin{equation}
 \mathcal R^*(H,\hat Z)\leq\mathcal R^*(H),\qquad
 \mathcal R^*(X,\hat Z)=\mathcal R^*(X).
 \label{eq:representationrisk}
\end{equation}
The first inequality is strict exactly when, for a positive-probability set of $h$, no label maximizes the refined posteriors almost surely under $\hat Z\mid H=h$.
\end{proposition}
\emph{Proof sketch.} Let $v_y=P(Y=y\mid H,\hat Z)$ and $y_H$ be a Bayes label given $H$. The tower property gives $P(Y=y\mid H)=\mathbb E[v_y\mid H]$, hence
\begin{equation}
 \Delta_{\mathrm{repr}}:=\mathcal R^*(H)-\mathcal R^*(H,\hat Z)
 =\mathbb E\!\big[\max_y v_y-v_{y_H}\big]\geq0.
 \label{eq:refinementgapmain}
\end{equation}
Equality holds exactly when $y_H$ remains optimal almost surely; $(X,A(X))$ and $X$ carry the same information. Positive Bayes-risk reduction requires a decision distinction, not merely changed confidence. If $\hat Z$ is already determined by $H$, this gain is zero. Appendix~\ref{app:theoryextra} gives the full proof, including ties. For an analytical binary example, let a constant $H$ merge equally likely appraisal states with class-1 posteriors $0.8$ and $0.2$. Preserving this distinction lowers Bayes error from $0.5$ to $0.2$ (Appendix~\ref{app:theoryextra}).

For fixed $H$ and $\hat Z$, let $f_U$ be a learned classifier with risk $\mathcal R(f_U)=\mathcal R^*(U)+\epsilon_U$, where $\epsilon_U\geq0$. Substitution for $U=H$ and $(H,\hat Z)$ yields
\begin{equation}
 \mathcal R(f_H)-\mathcal R(f_{H,\hat Z})
 =\Delta_{\mathrm{repr}}+\epsilon_H-\epsilon_{H,\hat Z}.
 \label{eq:riskaccounting}
\end{equation}
Improvement occurs exactly when $\epsilon_{H,\hat Z}-\epsilon_H<\Delta_{\mathrm{repr}}$: appraisal can help through decision distinctions or lower readout error. Section~\ref{sec:mechanism} tests intermediate organization and scene correspondence through generic and shuffled appraisal. The decomposition separates representational opportunity from the readout that realizes it; Equation~\ref{eq:pathway} accesses appraisal through its constrained weighting function. Matched recognition and connection contrasts assess the resulting decision behavior. Appendix~\ref{app:theoryextra} derives the excess-risk posterior form; Appendix~\ref{app:uncertain} bounds appraisal-error propagation.
\section{Experiments}
\label{sec:experiments}
\setcounter{topnumber}{2}

\begin{table}[t]
\centering
\caption{\textbf{CA-MER accuracy (\%).} Conflict averages the video-aligned and audio-aligned subsets. The four trained models share their common-stage examples and optimization schedule. Complete output-validity results and protocol-reassessed external methods are in Appendix~\ref{app:additional}.}
\label{tab:camer}
\small
\setlength{\tabcolsep}{4.4pt}
\begin{tabularx}{\linewidth}{@{}l*{5}{>{\centering\arraybackslash}X}@{}}
\toprule
\VISTAHeaderRow
\VISTAFirstHeader{Method} & \VISTAHeader{Video-aligned} & \VISTAHeader{Audio-aligned} & \VISTAHeader{Conflict} & \VISTAHeader{Consistent} & \VISTALastHeader{Overall}\\
\midrule
Qwen2.5-Omni Base &49.0&60.0&54.5&68.0&59.0\\
Emotion-SFT &55.0&65.0&60.0&73.0&64.3\\
Generic-CoT-SFT &57.0&66.0&61.5&74.0&65.7\\
Modality-Gate-SFT &58.0&66.0&62.0&74.0&66.0\\
\VISTAFocusRow
\VISTAFirstFocus{\method{}} &\textbf{61.0}&\textbf{68.0}&\textbf{64.5}&\textbf{74.2}&\VISTALastFocus \textbf{67.7}\\
\bottomrule
\end{tabularx}
\end{table}
The design in Section~\ref{sec:method} makes a specific prediction: appraisal should help most when a cue admits competing emotional interpretations. We first test where the recognition gain occurs, then use correspondence and connection controls to examine what the interface contributes. Interventions and complementary tasks test the resulting account beyond the primary comparison.
\subsection{Evaluation design}
\paragraph{Tasks and evaluation modes.}
CA-MER provides video-aligned, audio-aligned, and consistent subsets \citep{han2025camer}. We normalize outputs and average the first two accuracies for conflict. On EmoMM, we evaluate common-stage checkpoints without adaptation \citep{sun2026emomm}. We adapt to CH-SIMS v2 and MELD for sentiment accuracy across train-defined conflict groups and seven-class emotion recognition, respectively \citep{liu2022chsims,poria2019meld}. On THERADIA, we separately evaluate frozen-backbone appraisal probes and downstream emotion-intensity regression \citep{fournier2024theradia}. Appendix~\ref{app:data} details the protocols and shared corpus sources.

\paragraph{Matched common training.}
The four trained models share Qwen2.5-Omni-7B, 8,000 examples, and the optimization schedule in Section~\ref{sec:learning}; Base is frozen. Generic-CoT-SFT and VISTA both average 360 target tokens. The gate and VISTA have 12.19M and 12.39M trainable parameters. Scores are aggregates; contrasts use unrounded values. Appendices~\ref{app:repro} and~\ref{app:cost} give the configuration and compute record.

\paragraph{Conflict specificity.}
For a comparator $B$, define
\begin{equation}
 S(B)=\underbrace{\acc_{\conf}(\method)-\acc_{\conf}(B)}_{\Delta_{\conf}(B)}
      -\underbrace{\acc_{\cons}(\method)-\acc_{\cons}(B)}_{\Delta_{\cons}(B)}.
 \label{eq:specificity}
\end{equation}
Positive $S$ means that improvement is larger on the observed conflict subsets than on consistent inputs. This descriptive contrast complements absolute accuracy; all group scores and differences appear in Appendix~\ref{app:specificityreading}.

\subsection{The improvement concentrates on conflicting evidence}
\label{sec:mainresults}
VISTA reaches \textbf{64.5\%} conflict and \textbf{67.7\%} overall accuracy (Table~\ref{tab:camer}). Against Modality-Gate-SFT, the 2.5-point conflict gain accompanies a 0.2-point consistency gain ($S=2.3$), locating the benefit where the streams disagree.

Against generic reasoning, visual- and audio-aligned gains of 4.0 and 2.0 points contrast with 0.2 on consistency, giving $S=2.8$; against Emotion-SFT, $S=3.3$. Figure~\ref{fig:specificity} separates these effects. Under the common normalized interface, MoSEAR and CHASE reach 61.5\% and 61.0\% conflict accuracy, respectively, placing VISTA 3.0 and 3.5 points higher. These protocol-reassessed scores are distinct from published results.

Emotion-SFT tests label training, Generic-CoT-SFT supplies reasoning at the same mean target length, and Modality-Gate-SFT tests learned weighting. The latter two nearly match VISTA on consistency but separate on conflict; gains in both alignment directions extend to cases favoring either stream. The next tests examine which properties of appraisal accompany this advantage.

\begin{figure}[!t]
\centering
\includegraphics[width=\linewidth]{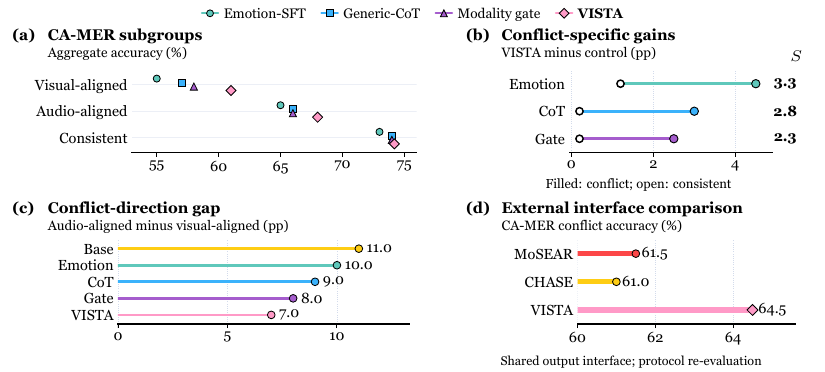}
\caption{\textbf{Conflict resolution beyond stronger reasoning and gating.} (a) Subgroup accuracy. (b) Conflict and consistent gains; $S$ is their difference. (c) The audio--video accuracy gap. (d) External methods re-evaluated under the common output interface. Points show aggregate scores; gains and gaps use percentage points.}
\label{fig:specificity}
\end{figure}

\begin{table}[!t]
\centering\small\setlength{\tabcolsep}{4.5pt}
\setlength{\abovecaptionskip}{0pt}
\caption{\textbf{Directed response and selective stability.} Direction is agreement with the annotated target-probability direction; class change counts label flips. Both use 150 relevant pairs; stability uses 150 irrelevant rewrites. Rates are percentages; $\Delta p$ and $\Delta$ log-odds retain native scales.}
\label{tab:pairedmain}
\begin{tabularx}{\linewidth}{@{}l*{5}{>{\raggedleft\arraybackslash}X}@{}}
\toprule
\VISTAHeaderRow
\VISTAFirstHeader{Method}&\VISTAHeader{Direction}&\VISTAHeader{Class change}&\VISTAHeader{$\Delta p$}&\VISTAHeader{$\Delta$ log-odds}&\VISTALastHeader{Stability}\\
\midrule
Emotion-SFT&62.0&24.0&0.050&0.201&91.0\\
Generic-CoT-SFT&66.0&28.0&0.070&0.281&90.0\\
Modality-Gate-SFT&63.0&25.0&0.055&0.221&91.0\\
\VISTAFocusRow
\VISTAFirstFocus{\method{}}&\textbf{75.0}&36.0&0.120&0.484&\VISTALastFocus \textbf{92.0}\\
\bottomrule
\end{tabularx}
\end{table}

\subsection{Scene correspondence and decision access strengthen recognition}
\label{sec:mechanism}
Higher accuracy alone does not establish why appraisal helps. Figure~\ref{fig:mechanism} therefore separates decision access, semantic content, and scene correspondence. Retaining appraisal supervision but removing its forward decision connection gives 62.5\% conflict accuracy, versus 64.5\% for full VISTA. A generic semantic bottleneck reaches 63.2\%, while renaming the appraisal fields retains 64.2\%. These controls associate the additional benefit with the organized appraisal input used by prediction.

Scene correspondence provides a further distinction. Cross-sample shuffling gives 62.0\%, whereas restricting donors to the same emotion class gives 63.7\%. The latter preserves the donor's emotion category but loses 0.8 points, supporting information tied to the particular event. Removing expression regulation gives 63.6\%. Removing appraisal, conflict, or reliability supervision gives 63.2\%, 63.5\%, or 63.0\%, respectively; Table~\ref{tab:fullablations} retains every variant.

The interface also differs from a displayed explanation. Editing the external rationale with internal state fixed leaves accuracy at 64.5\%; corrupting internal appraisal while preserving the explanation gives 62.0\%. Native visual/audio weights change from 0.42/0.30 on video-aligned inputs to 0.25/0.47 on audio-aligned inputs. These conditional means complement the replacement tests; together, the contrasts connect the representation analysis in Section~\ref{sec:representation} to a scene-matched input that participates in prediction.

\subsection{Relevant changes elicit directed, selective responses}
\label{sec:counterfactual}
The preceding tests remove or disrupt information. A complementary test asks whether changing relevant meaning produces a directed response while irrelevant reformulations preserve the decision. Table~\ref{tab:pairedmain} reports 150 relevant intervention pairs and 150 separate irrelevant-rewrite pairs. Direction agreement concerns the target probability; label change concerns the predicted class.

The reported results combine 75.0\% direction agreement with 92.0\% rewrite stability. The 36.0\% label-change rate distinguishes directed probability movements from those large enough to cross a class boundary.

Input masking gives a $+0.040$ target-modality diagnostic shift (controls: $+0.015$--$+0.020$), distinct from native $\alpha$. These are behavioral evaluations of the learned interface: the final configuration uses $\lambda_{\mathrm{cf}}=0$. Appendix~\ref{app:cfreconciliation} gives the complete paired record.

\begin{figure}[!t]
\centering
\includegraphics[width=\linewidth]{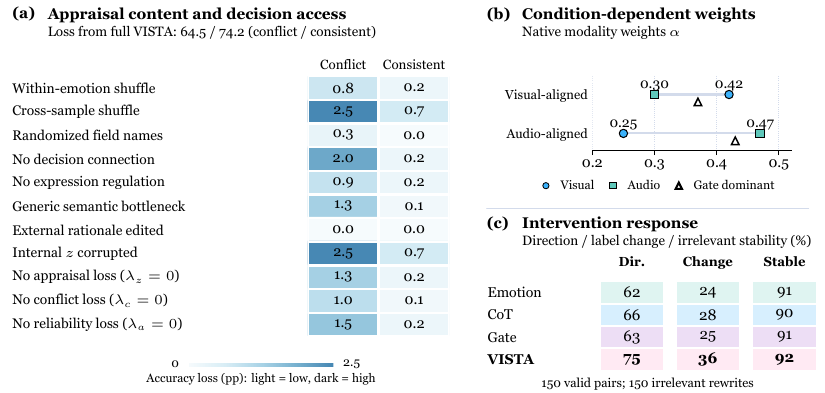}
\caption{\textbf{From scene-specific appraisal to selective responses.} (a) Accuracy losses under appraisal and training variants. (b) Native modality weights; triangles show the gate's dominant-modality weight. (c) Direction agreement, label changes, and irrelevant-rewrite stability on two separate 150-pair sets. Native weights are distinct from masking-derived diagnostics.}
\label{fig:mechanism}
\end{figure}

\begin{table}[!t]
\centering\small\setlength{\tabcolsep}{5pt}
\caption{\textbf{Recognition across conditions and tasks (\%).} EmoMM uses no benchmark-specific adaptation; CH-SIMS v2 and MELD use task-adapted models (Base is zero-shot). Acc2 is binary accuracy, Q4 the highest-conflict group, and wF1 weighted F1.}
\label{tab:recognitionmain}
\begin{tabularx}{\linewidth}{@{}l>{\raggedleft\arraybackslash}X>{\hsize=1.5\hsize\raggedleft\arraybackslash}X*{3}{>{\hsize=.833333\hsize\raggedleft\arraybackslash}X}@{}}
\toprule
&\multicolumn{2}{c}{EmoMM accuracy}&\multicolumn{2}{c}{CH-SIMS v2 Acc2}&MELD\\
\cmidrule(lr){2-3}\cmidrule(lr){4-5}
\VISTAHeaderRow
\VISTAFirstHeader{Method}&\VISTAHeader{Conflict}&\VISTAHeader{Conflict + missing}&\VISTAHeader{Overall}&\VISTAHeader{Q4}&\VISTALastHeader{wF1}\\
\midrule
Qwen2.5-Omni Base&46.5&38.0&79.4&70.0&59.55\\
Emotion-SFT&47.5&38.0&84.0&77.0&65.49\\
Generic-CoT-SFT&48.0&38.6&84.3&77.8&65.69\\
Modality-Gate-SFT&49.0&40.0&84.8&79.2&65.74\\
\VISTAFocusRow
\VISTAFirstFocus{\method{}}&\textbf{52.0}&\textbf{43.5}&\textbf{85.9}&\textbf{81.5}&\VISTALastFocus \textbf{66.94}\\
\bottomrule
\end{tabularx}
\end{table}

\subsection{The pattern extends across conflict strength and evaluation modes}
\label{sec:generalization}
Table~\ref{tab:recognitionmain} and Figure~\ref{fig:transfer} extend the evaluation to increasing conflict strength, missing evidence, and ordinary emotion recognition.

\paragraph{Conflicting and absent evidence.}
EmoMM conflict and conflict-plus-missingness accuracy reach 52.0\% and 43.5\%, gains of 4.0 and 4.9 points over Generic-CoT-SFT. The consistency-to-conflict drop is 4.3 points, versus 6.6--7.5 for trained controls. Appendix~\ref{app:additional} gives all conditions and the CHASE comparison, whose router used EmoMM supervision.

\Needspace{5\baselineskip}
\paragraph{Greater conflict, greater benefit.}
On CH-SIMS v2, the gain over Emotion-SFT grows from 0.2 points in Q1 to 4.5 in Q4. Q4 accuracy reaches 81.5\% and MAE falls from 0.365 to 0.315. This progression strengthens the conflict-specific account across training-defined intensity groups (Figure~\ref{fig:transfer}a).

\begin{figure}[!t]
\centering
\includegraphics[width=.94\linewidth]{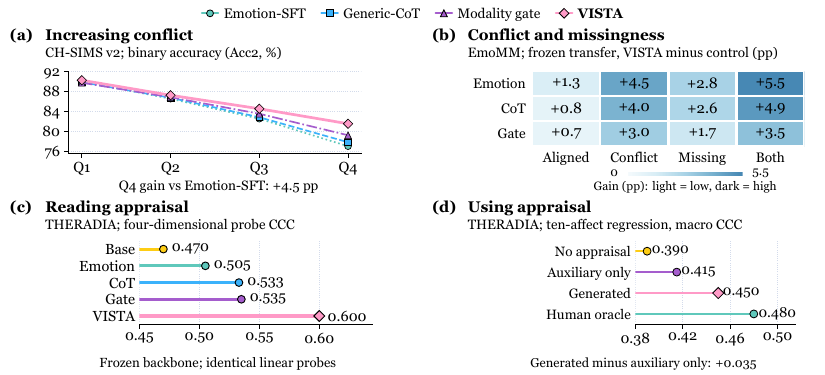}
\caption{\textbf{Conflict strength, frozen evaluation, and appraisal readout.} (a) Task-adapted CH-SIMS v2 binary accuracy by conflict group. (b) EmoMM gains without benchmark-specific adaptation. (c) THERADIA four-dimensional probe CCC on frozen backbones. (d) Ten-label emotion-intensity regression CCC with different appraisal access; human appraisal is an oracle.}
\label{fig:transfer}
\end{figure}

\paragraph{From readable appraisal to useful appraisal.}
Identical frozen-backbone probes yield THERADIA macro CCC 0.600 for VISTA versus 0.505 for Emotion-SFT. Downstream ten-emotion CCC is 0.390 without appraisal, 0.415 with disconnected auxiliary supervision, 0.450 with generated appraisal, and 0.480 with human appraisal. The 0.035 gain over auxiliary-only training complements the CA-MER connection contrast: readable appraisal also supports prediction (Appendix~\ref{app:learnabilityuse}).

\paragraph{Recognition scope and practical cost.}
MELD weighted-F1 improves by 1.46 points over Emotion-SFT, with higher recall in all seven classes (Table~\ref{tab:recognitionmain}). VISTA uses 72 GPU-hours per seed and 15.5\,s median latency, versus 40 hours/13.5\,s for equal-length Generic-CoT and 64 hours/6.0\,s for gating. Appendices~\ref{app:cost} and~\ref{app:ordinary} retain full resources and class results.

\paragraph{The role of learning.}
\label{sec:frozen}
Frozen appraisal prompting gives 56.0\% conflict accuracy versus 56.4\% for generic reasoning; training reverses this ordering (64.5\% versus 61.5\%), supporting the learned interface. Appendices~\ref{app:additional} and~\ref{app:annotation} retain the prompting and supervision audits.
\Needspace{8\baselineskip}
\section{Discussion and Conclusion}
\label{sec:conclusion}
VISTA connects what an event means to a person with how its emotional evidence is used. Distinguishing emotion expectation from cue diagnosticity motivates a concern-relative appraisal that reweights modality logits alongside joint evidence. Recognition gains concentrate on conflict; correspondence, connection, and intervention tests connect this benefit to the scene-specific interface. The analysis explains when reweighting can alter a class preference and when appraisal can refine a compressed representation. Together, these results establish a concrete route from scene meaning to conflict-sensitive emotion recognition.
\FloatBarrier
\setlength{\textfloatsep}{\VISTADefaultTextFloatSep}
\begingroup
\titlespacing*{\subsection}{0pt}{2ex plus .5ex}{1ex plus .2ex}
\subsection*{AI use statement}
We used AI assistants in two roles. First, for manuscript preparation, including grammar checking, text polishing, and assistance with conceptual framing, mathematical arguments, result interpretation, and figure preparation. Second, for routine coding assistance during implementation, including annotation processing and visualization. The illustrative scene in Figure~\ref{fig:scene} was AI-generated; model-generated appraisal supervision is described in Section~\ref{sec:learning}. The method, experimental design, and all implementation decisions affecting the reported results were finalized by the authors, who verified all AI-assisted output and take full responsibility for this paper.

\subsection*{Ethics statement}
Appraisal is a contextual hypothesis about an event, rather than a durable personal attribute or a judgment of moral worth. Goals, norms, and expression vary across people and settings. Consequential uses require context-specific validation and human review; conversational and audiovisual data require attention to participant privacy and dataset access conditions. Appendix~\ref{app:repro} discusses responsible interpretation.

\subsection*{Reproducibility statement}
Appendix~\ref{app:proof} gives the analytical assumptions and full derivations. Appendix~\ref{app:data} describes datasets and conflict protocols, Appendix~\ref{app:annotation} records supervision, and Appendix~\ref{app:pilot} specifies common training and task adaptation. Appendix~\ref{app:repro} maps the complete experimental evidence. The appendices report the complete experimental results, scoring conventions, and links between the analyses, tables, and figures. For the independent public-label prevalence audit, Appendix~\ref{app:conflictdefinition} specifies the label mapping, exact comparison rule, sample counts, and data provenance.
\par\endgroup
\clearpage
\setcounter{topnumber}{2}

\clearpage
\appendix
\renewcommand{\floatpagefraction}{0.85}
\VISTAAppendixHeadings
\VISTAAppendixFront

\section{Reading the seven appraisal fields}
\label{app:semantic}
The appraisal schema gives the model a reference for interpreting an event: what the person seeks to attain or protect, how the outcome bears on that concern, and how the situation shapes its expression. Its usefulness depends on preserving these relations. An outcome description such as ``the candidate was not selected'' records what happened. Goal congruence additionally asks whether that outcome obstructed something the candidate wanted. A smile records an expression; expression regulation concerns the conditions under which that expression was produced. This distinction between an observation and its event-relative interpretation organizes the seven fields.

\subsection{Field types and missing supervision}
\label{app:appraisalschema}
Each field stores four components: \texttt{value}, \texttt{confidence}, \texttt{evidence}, and \texttt{mask}. Table~\ref{tab:appraisaltypes} specifies the value space. Evidence anchors connect a field to the input; confidence records the certainty of that judgment, while the mask determines whether it supplies a valid supervision target. Confidence is recorded separately from supervision weights.

\begin{table}[htbp]
\caption{\textbf{Typed appraisal interface.} These value spaces distinguish scene semantics, expressive conditions, and missing information. Each norm bit has its own validity mask.}
\label{tab:appraisaltypes}
\small
\begin{tabularx}{\linewidth}{@{}>{\raggedright\arraybackslash}p{1.1cm}>{\raggedright\arraybackslash}p{2.6cm}Y@{}}
\toprule
\VISTAHeaderRow
\VISTAFirstHeader{Field} & \VISTAHeader{Type} & \VISTALastHeader{Values}\\
\midrule
$g$ & Short text & Goal or concern supported by the scene\\
$c$ & Categorical & \texttt{promotes}, \texttt{obstructs}, \texttt{irrelevant}\\
$e$ & Categorical & \texttt{expected}, \texttt{violated}, \texttt{uncertain}\\
$a$ & Categorical & \texttt{self}, \texttt{other}, \texttt{environment}, \texttt{shared}, \texttt{unknown}\\
$k$ & Continuous & $[0,1]$ coping/control value\\
$n$ & Four binary entries & \texttt{politeness}, \texttt{identity}, \texttt{status}, \texttt{relationship}\\
$r$ & Categorical & \texttt{none}, \texttt{suppressed}, \texttt{masked}, \texttt{exaggerated}, \texttt{social\_maintenance}\\
\bottomrule
\end{tabularx}
\end{table}

A missing field is stored as \texttt{null} and excluded by its mask; it is not replaced with numerical zero. The output state $a=\texttt{unknown}$ is retained when applicable but does not provide a valid agency target. Each bit of $n$ is set to \texttt{false} only when the input supports a negative judgment; absent evidence leaves that bit missing. Thus, missing coping information is distinct from $k=0$, and an unresolved norm is distinct from an explicitly absent one. Appraisal losses follow these types and normalize over valid supervision, as specified in Appendix~\ref{app:implementation}.

\subsection{The unit of appraisal and its evidence}
An appraisal is indexed by a recognition target, an event, and the time at which the response is interpreted. Write this unit as $(p,\omega,\tau)$, where $p$ is the person whose affect is being recognized and $\omega$ is the relevant outcome or interaction. This indexing supplies a semantic contract for Equation~\ref{eq:schema}; it is not an additional model input. In a selection scene, the candidate's unsuccessful application and the colleague's successful application concern the same announcement but different target--event relations. After an appeal becomes possible, the candidate's control can change even though the original outcome does not.

The fields answer different types of question. The concern $g$ identifies an object to attain or protect. Congruence $c$ evaluates the relation between that object and $\omega$. Expectation $e$ concerns what was anticipated before $\omega$; agency $a$ concerns its attributed cause. Control $k$ concerns available responses after or during the event. Norms $n$ identify standards relevant to the event or interaction, and regulation $r$ concerns how the resulting response is displayed. These semantic types explain why the seven entries cannot be replaced by seven interchangeable positive--negative scores.

For example, let $\mathcal O$ be the event outcomes described in the available context and let $\succeq_{p,g}$ denote the target's stated or context-supported preference ordering relative to concern $g$. Comparing the observed outcome $o\in\mathcal O$ with the pertinent alternative $o_0$ gives a conceptual reading of congruence:
\begin{equation}
 o\succ_{p,g}o_0:\ \text{supports }g,\qquad
 o\prec_{p,g}o_0:\ \text{obstructs }g.
 \label{eq:congruenceorder}
\end{equation}
The alternative $o_0$ is fixed by the comparison, such as receiving versus not receiving the position. Holding that comparison fixed while changing the concern isolates its evaluative role. An unestablished preference relation leaves the judgment unresolved; it does not imply indifference. This ordering is a semantic device rather than an implemented utility function. It makes the role of concern precise: changing $g$ can change the ordering while the event fact $o$ remains fixed.

\paragraph{A linked interpretation rather than seven independent observations.}
An evidence anchor and the inference made from it play different roles. ``I applied because I wanted to lead the project'' supports the concern; ``the position went to someone else'' establishes the outcome; their relation supports obstruction. A final-decision announcement also bears on immediate control, while the presence of the selected colleague bears on expressive obligations. These conclusions reuse observations, so agreement among fields is not independent corroboration. Their value is in making the dependencies explicit: congruence must refer to the stated concern, agency to the appraised event, and regulation to the observed display under the relevant interaction conditions. A coherent goal edit can therefore require a congruence edit, as developed in Appendix~\ref{app:interventionguide}.

\subsection{Evidence anchors and inferential distinctions}
Table~\ref{tab:semanticanchors} describes how to read each field from available context. An evidence anchor is material in the scene that bears on the question; it need not directly state the appraisal. For instance, a prior request to join a committee can support an inference about the relevant goal, while the committee's decision establishes the outcome. Relating those two observations supplies the congruence judgment.

\begin{table}[htbp]
\caption{A semantic reading guide for the seven fields. Evidence anchors identify relevant observations; the final column specifies the relation to interpret.}
\label{tab:semanticanchors}
\small
\begin{tabularx}{\linewidth}{@{}>{\raggedright\arraybackslash}p{2.25cm}YY@{}}
\toprule
\VISTAHeaderRow
\VISTAFirstHeader{Field} & \VISTAHeader{Evidence anchor} & \VISTALastHeader{Distinction to preserve}\\
\midrule
Goal / concern & Requests, commitments, prior choices, or stated priorities in the available context & The current object of concern, such as obtaining this role, is more specific than a general value such as achievement.\\
Goal congruence & The observed outcome together with evidence about the relevant goal & Whether the event supports or obstructs the goal; the same external outcome can have different signs for different concerns.\\
Expectation & Prior plans, predictions, promises, or established patterns & Expectedness concerns anticipation. An undesired event can be expected, and a desirable event can be surprising.\\
Agency & Actions, decisions, attributions, and identified participants & Who or what produced the outcome; identifying a responsible agent does not by itself establish intent or blame.\\
Coping / control & Available remedies, remaining choices, resources, or finality of the decision & The person's capacity to alter or manage the consequences, distinct from responsibility for causing them.\\
Norm / social relevance & Roles, relationships, agreed procedures, audience, and obligations & Which interpersonal standards bear on this situation, including whether an outcome or an expression is socially expected.\\
Expression regulation & Expression in relation to the event, audience, and other cues & Whether suppression, masking, or exaggeration helps explain the display; a visible smile alone does not establish masking.\\
\bottomrule
\end{tabularx}
\end{table}

\paragraph{Binding appraisal to the person and event.}
A useful concern is specific enough to relate to the outcome. ``Protecting an important relationship'' becomes informative when the relationship and the threatened interaction are identifiable. The same scene can contain several people with different concerns, so a successful outcome for one person need not be successful for the recognition target. It can also involve more than one concern for that person: losing a competition may obstruct achievement while an appropriate congratulation preserves a relationship. The schema organizes this coexistence without requiring a one-to-one mapping from a value category to an emotion.

\paragraph{Keeping related fields distinct.}
Agency and control occupy different positions in an account of the event. A committee may make a decision that the candidate can appeal, or make a decision that is final. Conversely, the candidate may have caused a problem that another person now controls. Norms and regulation are similarly related but distinct. A norm can judge the event itself, such as whether a selection procedure was fair, or prescribe an appropriate response. A public occasion can create an obligation to be courteous; whether a particular display serves that obligation is a further question about regulation. Separating these questions prevents a single contextual fact from being counted as several independent reasons for an emotion label.

\Needspace{9\baselineskip}
\subsection{Underspecified scenes and unresolved appraisals}
A short clip may reveal the outcome while omitting the concern, the expectation, or the available response. In this semantic guide, an unresolved field remains an open question. Lack of evidence for obstruction does not imply that a goal was supported, and lack of a visible attempt to intervene does not establish low control. Likewise, an unspecified expectation cannot be filled by assuming that every negative outcome was surprising. The missing-value and validity rules in Appendix~\ref{app:appraisalschema} preserve these distinctions in the training targets.

The other fields can remain informative when one question is unresolved. A clear announcement of a final decision can support low immediate control even when the candidate's expectations are unknown. A formal public exchange can establish an expressive obligation without establishing the private emotional response. Such partial interpretations preserve useful evidence while keeping the final affective judgment responsive to the full multimodal input.

\subsection{Appraisal interfaces and their decision roles}
\label{app:relatedinterfaces}
Table~\ref{tab:positioning} complements Section~\ref{sec:related} by distinguishing the information organized by each interface, the decision it informs, and the corresponding evaluation target.

\begin{table}[t]
\centering
\caption{\textbf{Related approaches by decision role.} Selected methods organized by the information used, the decision it informs, and the evaluation target.}
\label{tab:positioning}
\small
\setlength{\tabcolsep}{4pt}
\begin{tabularx}{\linewidth}{@{}>{\raggedright\arraybackslash}p{2.05cm}YYY@{}}
\toprule
\VISTAHeaderRow
\VISTAFirstHeader{Approach} & \VISTAHeader{Organizing signal} & \VISTAHeader{Decision role} & \VISTALastHeader{Evaluation focus}\\
\midrule
\citet{wang2026appraisal} & Expectations and their violations & Predict emotion labels and shifts & Emotion labels and transition dynamics\\
\addlinespace[3pt]
AG-CTR$^2$ \citep{chu2026appraisal} & Event understanding, appraisal, and coping & Retrieve support using appraisal chains & Support generation and retrieval-query quality\\
\addlinespace[3pt]
CHASE \citep{sun2026emomm} & Modality hidden states and attention preferences & Detect conflict and steer head-level attention & Recognition under conflict and missingness\\
\addlinespace[3pt]
\VISTAFocusRow
\VISTAFirstFocus{\method{}} & Concerns, event relations, and expression conditions & Condition modality arbitration and affect prediction &\VISTALastFocus  Conflict gains, correspondence, and decision access\\
\bottomrule
\end{tabularx}
\end{table}
\section{Paired scenes for concern-relative interpretation}
\label{app:scenes}
These constructed analytical illustrations examine three distinct roles of appraisal: identifying what an outcome means, interpreting a public display, and distinguishing disappointment from blame. They are not benchmark samples, model outputs, or members of the counterfactual evaluation. Contextual premises are supplied explicitly; the resulting interpretations are plausible rather than uniquely determined emotion labels.
\begin{figure}[htbp]
\centering
\includegraphics[width=\linewidth]{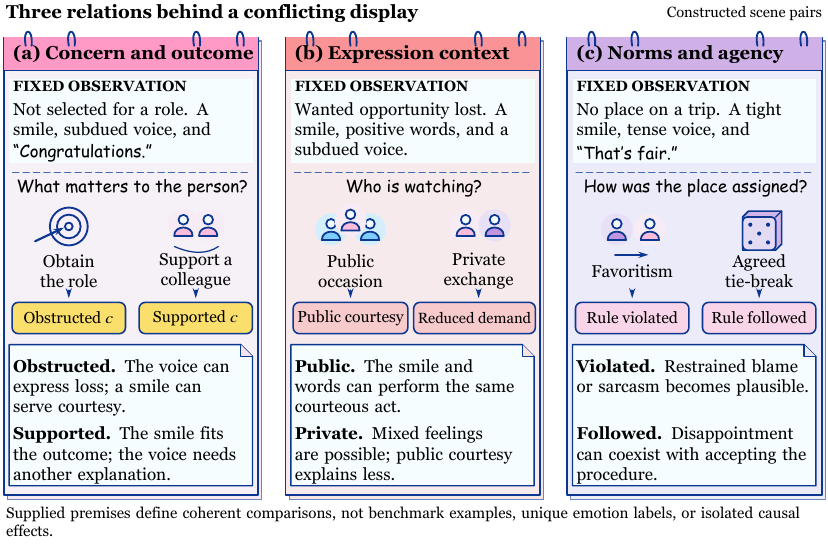}
\caption{\textbf{The context changes what a cue can mean.} Each constructed pair holds the described surface observation fixed while changing the goal, audience, or allocation rule. The contrasts illustrate appraisal relations; they are not evaluated predictions.}
\label{fig:scene_pairs}
\end{figure}
\subsection{The same outcome can support different concerns}
A candidate is not selected for a role. They smile, say ``Congratulations,'' and speak in a subdued voice. If they actively sought the position, the outcome obstructs their goal; courtesy can explain why the positive display coexists with disappointment. If they entered to support a colleague and preferred to avoid the responsibility, the same outcome supports their stated purpose. The subdued voice then needs another explanation, such as fatigue, if the context supports it. Goal congruence thus relates an event to a concern: ``not selected'' alone does not determine whether the outcome is unfavorable. Identifying the concern changes the coherent explanations of the cues even when the final emotion remains uncertain.

\subsection{The same setback can be expressed under different social demands}
After losing a wanted opportunity, a person smiles and says ``I'm happy for them'' in a subdued voice. In front of the selected colleague and an audience, the smile and words can perform the same courteous act. Their agreement need not outweigh the subdued prosody: two cues may share an expressive purpose. In a private conversation with a trusted friend, that public demand is reduced. The same words might express happiness for the colleague alongside personal disappointment, or an attempt at self-regulation. The setting changes the available explanations without choosing one automatically. Norm and regulation fields connect the setback to the conditions under which its response becomes observable.

\subsection{Agency and norms distinguish disappointment from blame}
A person denied a wanted place on a trip says ``That's fair'' with a tight smile and tense voice. If a manager bypassed an agreed procedure to favor a friend, sarcasm or restrained blame becomes plausible. If an agreed random tie-break was followed, the sentence can acknowledge the rule while the voice expresses disappointment. Agency locates responsibility; norms supply the standard for evaluating the allocation. Neither replaces goal congruence: a fair procedure can still yield an unwanted outcome. This pair changes agency and norms together, illustrating a coherent scene contrast rather than a single-coordinate intervention.
\Needspace{7\baselineskip}
\section{Cue diagnosticity and representational interpretation}
\label{app:proof}
\subsection{Emotion expectation and cue diagnosticity}
This section develops the log-odds decomposition in Equation~\ref{eq:semanticlogodds}. For two affective hypotheses $y_1,y_0$, a cue event $E=e$, and appraisal $Z=z$, Bayes' rule gives, whenever the conditioning probabilities and denominators are positive,
\begin{equation}
 \frac{P(y_1\mid e,z)}{P(y_0\mid e,z)}
 =\underbrace{\frac{P(y_1\mid z)}{P(y_0\mid z)}}_{\text{emotion expectation}}
 \underbrace{\frac{P(e\mid y_1,z)}{P(e\mid y_0,z)}}_{\text{cue diagnosticity}}.
 \label{eq:diagnosticityodds}
\end{equation}
The first factor concerns which emotion fits the event appraisal. The second concerns how much the cue distinguishes the two emotions within that appraisal. A blocked goal can make disappointment more plausible; a courtesy obligation can make a smile compatible with disappointment, changing the smile's diagnostic role. Improving the first factor alone would not demonstrate improved cue interpretation. These semantic roles do not map one-to-one onto the branches in Equation~\ref{eq:pathway}: Appraisal-conditioned weights can reflect both expected affect and changing cue diagnosticity, while the joint residual preserves evidence beyond the schema. Path interventions locate a computational dependency; identifying its semantic role additionally requires examining how cue use changes with appraisal.

Equation~\ref{eq:diagnosticityodds} is a conditional-probability identity, not a likelihood model fitted by VISTA. It assumes no independence between modalities and supplies no causal effect. Here $z$ is a conceptual conditioning state; the predicted appraisal $\hat z=A(X)$ can itself depend on the cue, so its first factor is not an estimate made with that cue excluded. It adds no independent observation beyond $X$. The following analysis characterizes possible uses of this information, without assuming that VISTA implements a Bayes module.

\paragraph{How a fixed cue can cross the decision boundary.}
Write $b(z)$ for the log prior odds, $\ell(e,z)$ for the log likelihood ratio, and $L(e,z)$ for the log posterior odds. Equation~\ref{eq:diagnosticityodds} becomes
\begin{align}
 L(e,z)&=b(z)+\ell(e,z),\nonumber\\
 L(e,z')-L(e,z)&=b(z')-b(z)+\ell(e,z')-\ell(e,z).
 \label{eq:logoddsdecomposition}
\end{align}
For binary zero--one decisions, $y_1$ is preferred exactly when $\ell(e,z)>-b(z)$. Thus even at a fixed emotional expectation, a change in cue likelihood can move the posterior across the boundary. Conversely, a prior change can move the decision while leaving cue diagnosticity fixed. The two mechanisms produce the same kind of final-label change but explain it differently.

Figure~\ref{fig:appraisal_geometry} gives exact arithmetic for a constructed example. Let $y_1$ be satisfaction, $y_0$ disappointment, and $e$ a positive display. Two expression conditions, $z_f$ (ordinary expression) and $z_c$ (courtesy), share $P(y_1\mid z)=0.4$ and $P(e\mid y_1,z)=0.8$. Set $P(e\mid y_0,z_f)=0.2$ and $P(e\mid y_0,z_c)=0.6$. Then
\begin{align}
 P(y_1\mid e,z_f)&=\frac{0.4\cdot0.8}{0.4\cdot0.8+0.6\cdot0.2}=\frac{8}{11},\nonumber\\
 P(y_1\mid e,z_c)&=\frac{0.4\cdot0.8}{0.4\cdot0.8+0.6\cdot0.6}=\frac{8}{17}.
 \label{eq:displaytoy}
\end{align}
The display remains more likely under satisfaction in both conditions, but its likelihood ratio decreases from $4$ to $4/3$. With prior odds $2/3$, this is enough to change the preferred hypothesis. The example illustrates reduced diagnosticity rather than declaring a courteous display intrinsically negative.

\begin{figure}[htbp]
\centering
\includegraphics[width=\linewidth]{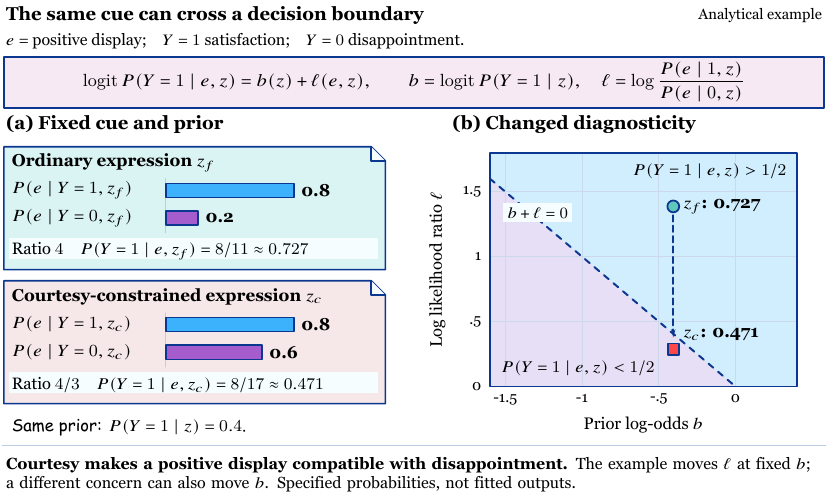}
\caption{\textbf{A context-dependent cue in a constructed probability model.} Both conditions have prior satisfaction probability $0.4$ and positive-display likelihood $0.8$ under satisfaction. The likelihood under disappointment changes from $0.2$ to $0.6$, giving posteriors $8/11$ and $8/17$. The log-odds plane separates emotional expectation from cue diagnosticity; its boundary is $b+\ell=0$. These are exact analytical calculations, not model predictions or fitted probabilities.}
\label{fig:appraisal_geometry}
\end{figure}

\paragraph{A cue contrast removes a purely contextual prior shift.}
At a fixed appraisal, compare two cue events $e_1,e_0$ and define
\begin{equation}
 T(z)=L(e_1,z)-L(e_0,z)=\ell(e_1,z)-\ell(e_0,z).
 \label{eq:cuecontrast}
\end{equation}
Subtracting again across appraisals gives $T(z')-T(z)$, the mixed contrast $\mathcal I$ in Equation~\ref{eq:cueinteractionmain}, in which both prior terms cancel. Under this probability model, a nonzero contrast establishes that the relative contribution of the two cues depends on appraisal; a context-only additive prior shift cannot produce it. This is a sharper semantic question than whether the final prediction changes after replacing appraisal. In an empirical test, the cue and appraisal contrasts would also need a common log-odds response scale and coherent conditioning states. The reported aggregate interventions do not supply this four-condition cue comparison.

\paragraph{Additive scores and cue--appraisal interactions.}
The contrast also characterizes the restriction imposed by an additive score. On a product domain of cue and appraisal values with positive conditional probabilities, fix a reference pair $(u_0,z_0)$. Define
\begin{align}
 s(u)&=L(u,z_0)-L(u_0,z_0),&t(z)&=L(u_0,z),\nonumber\\
 J(u,z)&=L(u,z)-L(u_0,z)-L(u,z_0)+L(u_0,z_0).
 \label{eq:anchoredinteraction}
\end{align}
Direct substitution yields
\begin{equation}
 L(u,z)=s(u)+t(z)+J(u,z).
 \label{eq:interactiondecomposition}
\end{equation}
Thus all anchored contrasts vanish if and only if the log-odds surface is additively separable. A nonzero $J$ cannot be absorbed into any context-only offset while keeping the cue score context-independent. The need for an interaction concerns reproducing the conditional score: an additive model can still predict the same labels on some examples despite having different scores. The product-domain assumption matters because all four conditional states must be defined; unsupported cue--scene combinations do not supply an identifiable contrast.

This gives a design reason to expose evidence and appraisal jointly to a predictor. Equation~\ref{eq:pathway} permits that dependence through appraisal-conditioned modality weights, while $D(H,\alpha)$ preserves the joint-evidence residual. The argument does not select a unique architecture or assign prior and diagnosticity terms to separate neural branches. With predicted appraisal, an empirical four-condition test would additionally distinguish recomputing $A(X)$ after a cue change from replacing its output while fixing the evidence.

The same constructed model makes this contrast explicit without introducing additional parameters. Take $e_1$ to be a positive display and $e_0$ its absence. Under ordinary expression, the two likelihood ratios are $4$ and $(1-0.8)/(1-0.2)=1/4$; under courtesy they are $4/3$ and $(1-0.8)/(1-0.6)=1/2$. Consequently,
\begin{equation}
 T(z_f)=\log 16,\qquad T(z_c)=\log(8/3),\qquad
 T(z_c)-T(z_f)=-\log 6.
 \label{eq:constructedcuecontrast}
\end{equation}
Courtesy reduces the contrast between displaying and withholding a positive response, while the contextual prior remains fixed. These exact quantities concern the constructed distribution in Equation~\ref{eq:displaytoy}; they are not measured effects of VISTA's interventions.

\paragraph{Agreement between modalities need not multiply the evidence.}
For two observed cues $e_t,e_v$, the exact joint likelihood ratio factors as
\begin{equation}
 \frac{P(e_t,e_v\mid y_1,z)}{P(e_t,e_v\mid y_0,z)}
 =\frac{P(e_t\mid y_1,z)}{P(e_t\mid y_0,z)}
  \frac{P(e_v\mid e_t,y_1,z)}{P(e_v\mid e_t,y_0,z)}.
 \label{eq:conditionalcueevidence}
\end{equation}
The second factor measures what the visual cue adds after the text cue is known. If a smile and a congratulation serve a shared courtesy act, they can be strongly dependent given the affective hypothesis and appraisal. Multiplying their marginal likelihood ratios can then double-count shared information. Equation~\ref{eq:conditionalcueevidence} neither privileges text nor assumes a processing order: the chain rule can be written in either order. It explains why agreement among modalities and independent evidence for an emotion are different properties.

\subsection{Uncertain and imperfect appraisal}
\label{app:uncertain}
A partially specified event can support several appraisal hypotheses. To make their use explicit, consider an analytical setting with residual scene evidence $W=w$, a cue $E=e$, and finitely many states $z$. Define the pre-cue weights $\rho_z=P(z\mid w)$ and posterior weights $\pi_z=P(z\mid e,w)$. The law of total probability gives
\begin{align}
 p(y\mid e,w)&=\sum_z K_z(y)\pi_z,
 &K_z(y)&=P(y\mid e,z,w),\nonumber\\
 \pi_z&=\frac{P(e\mid z,w)\rho_z}{\sum_{z'}P(e\mid z',w)\rho_{z'}}.
 \label{eq:appraisalmarginal}
\end{align}
Uncertainty is therefore averaged at the probability level, with weights conditioned on the evidence being explained. Averaging log odds, taking the most likely state first, or averaging with pre-cue weights generally gives another prediction.

For example, assign equal pre-cue weight to the two constructed states in Equation~\ref{eq:displaytoy}. Their display probabilities are $0.44$ and $0.68$, so observing the display changes the appraisal weights to $11/28$ and $17/28$. The resulting satisfaction probability is
\begin{equation}
 \frac{11}{28}\frac{8}{11}+\frac{17}{28}\frac{8}{17}=\frac47\simeq0.5714.
 \label{eq:toymarginal}
\end{equation}
An unweighted average of the two conditional posteriors would instead be approximately $0.5989$. The calculation shows exactly where evidence about expression conditions enters a coherent uncertain interpretation. It does not prescribe a posterior representation or marginalization procedure for the implemented model.

\begin{proposition}[Propagation of appraisal uncertainty]
\label{prop:appraisalerror}
In the finite-state analytical model of Equation~\ref{eq:appraisalmarginal}, let $p=\sum_z\pi_zK_z$ and $\tilde p=\sum_z\tilde\pi_z\tilde K_z$ be two affect distributions. With total variation $\operatorname{TV}(u,v)=\tfrac12\sum_j|u_j-v_j|$, define
\begin{equation}
 \delta=\operatorname{TV}(\pi,\tilde\pi),\qquad
 \eta=\sum_z\tilde\pi_z\operatorname{TV}(K_z,\tilde K_z).
 \label{eq:appraisalerrors}
\end{equation}
Then $\operatorname{TV}(p,\tilde p)\leq\delta+\eta$. If $p$ has maximizing label $y^*$ with margin $\gamma=p(y^*)-\max_{y\ne y^*}p(y)$, its decision is preserved whenever $\gamma>2(\delta+\eta)$.
\end{proposition}
\begin{proof}
Insert $\sum_z\tilde\pi_zK_z$ between the two mixtures and apply the triangle inequality. Since every $K_z$ sums to one,
\begin{align}
 \operatorname{TV}(p,\tilde p)
 &\leq\frac12\sum_y\left|\sum_z(\pi_z-\tilde\pi_z)K_z(y)\right|
       +\sum_z\tilde\pi_z\operatorname{TV}(K_z,\tilde K_z)\nonumber\\
 &\leq\frac12\sum_z|\pi_z-\tilde\pi_z|+\eta=\delta+\eta.
 \label{eq:appraisalerrorproof}
\end{align}
Each label probability changes by at most this total variation. Hence the gap between $y^*$ and any competitor decreases by at most $2(\delta+\eta)$, proving the margin statement.
\end{proof}
The two terms separate misallocating probability among scene interpretations from misreading the cues within an interpretation. When the kernels are shared, $\eta=0$: uncertain appraisal can be harmless to a high-margin decision and consequential near a boundary. A wrong point estimate can also matter little if its induced affect distribution resembles that of the correct state. Field agreement alone does not determine either term. This links the semantic audit to recognition: the relevant question is which appraisal errors alter the competing affective explanations, not how many field names differ.

\subsection{What a predicted representation can contribute}
\label{app:theoryextra}
Because $\hat Z=A(X)$ is computed from the recognition input, it introduces no new observation beyond $X$. Its potential benefit is representational: the schema can retain distinctions omitted by a compressed evidence representation $H=F(X)$, or make existing distinctions easier for a learned decision rule to use. Proposition~\ref{prop:representationrisk} makes the distinction precise for finite-label classification with zero--one loss. The continuous appraisal and affect-regression evaluations use their respective prediction losses and CCC metrics.

\paragraph{Assumptions and conditioning.}
The maps $F$ and $A$ are fixed when risks are evaluated, as after training. Conditional posteriors are understood up to probability-zero states; a fixed ordering of the finite labels resolves ties. The analysis imposes neither conditional independence among modalities nor a sufficiency assumption on $H$. In particular, it does not require the semantic appraisal to be correct: any additional deterministic representation obeys the Bayes-risk inequality, while its decision value depends on the distinctions it retains.

\begin{proof}[Proof of Proposition~\ref{prop:representationrisk}]
Write $v_y(h,z)=P(Y=y\mid H=h,\hat Z=z)$ and select a coarse Bayes label $y_h\in\arg\max_y P(Y=y\mid H=h)$. The tower property gives $P(Y=y\mid H=h)=\mathbb E[v_y(h,\hat Z)\mid H=h]$. Since the maximum of finitely many coordinates is convex,
\begin{equation}
 \max_y\mathbb E[v_y(h,\hat Z)\mid h]
 \leq\mathbb E[\max_y v_y(h,\hat Z)\mid h].
 \label{eq:conditionalmax}
\end{equation}
Averaging and subtracting from one proves the inequality. More explicitly, the exact gain is
\begin{equation}
 \Delta_{\mathrm{repr}}
 =\mathbb E\!\left[\max_y v_y(H,\hat Z)-v_{y_H}(H,\hat Z)\right].
 \label{eq:refinementgap}
\end{equation}
The integrand is nonnegative. The gain vanishes exactly when $y_H$ is also a maximizing label of the refined posterior almost surely. Equivalently, for almost every $h$, there is a label that maximizes the posteriors for almost every $z$ under $\hat Z\mid H=h$. If such a common maximizer fails to exist on a set of $h$ with positive probability, the conditional nonnegative gap is positive there and so is its expectation. This proves both the equality and strictness conditions, including posterior ties. Finally, $X$ and $(X,A(X))$ generate the same information, so their conditional label distributions, and hence Bayes risks, agree.
\end{proof}
For a concrete strict case, suppose $H=h_0$ is constant and two equiprobable appraisal states have class-1 probabilities $0.8$ and $0.2$. The coarse posterior is $0.5$ and its Bayes error is $0.5$; retaining the state gives error $0.2$. If the probabilities are instead $0.6$ and $0.9$, both states favor class 1 and both representations have error $0.25$. Posterior variation matters to classification exactly when it reveals a decision distinction. If $\hat Z$ is already determined by $H$, the common-maximizer condition holds automatically.

\paragraph{How much a collapsed appraisal distinction can cost.}
The strictness condition has a closed form for two binary-label appraisal states within a fixed evidence state $H=h$. Let $P(\hat Z=z_+\mid h)=q\in(0,1)$, and let $p_+>1/2>p_-$ be the class-1 posteriors in the two states. Write their opposing decision margins as $m_+=2p_+-1>0$ and $m_-=1-2p_->0$. The coarse posterior is $\bar p=qp_++(1-q)p_-$, so the conditional Bayes-risk reduction is
\begin{align}
 \Delta(h)
 &=\min\{\bar p,1-\bar p\}-q(1-p_+)-(1-q)p_-\nonumber\\
 &=\min\{q\,m_+,(1-q)m_-\}.
 \label{eq:binaryrefinementgain}
\end{align}
If the coarse classifier selects class 1, its excess conditional error occurs in state $z_-$ and equals $(1-q)m_-$. If it selects class 0, the corresponding cost is $qm_+$. The coarse Bayes rule chooses the smaller of these costs, which proves the formula and includes the case $\bar p=1/2$. The gain is therefore controlled jointly by the frequency and decision margin of the distinction being collapsed. In this two-state setting, averaging $\Delta(h)$ over $H$ recovers Equation~\ref{eq:refinementgapmain}; for $q=1/2$, $p_+=0.8$, $p_-=0.2$, the gain is $0.3$. If both states share an optimal label, the gain is zero instead, as in the $0.6/0.9$ example above. These are analytical distributions, not estimated appraisal posteriors.

\paragraph{From representation opportunity to trained risk.}
For any readout $f_U$, its excess zero--one risk has the posterior form
\begin{equation}
 \epsilon_U
 =\mathbb E\!\left[\max_y P(Y=y\mid U)-P(Y=f_U(U)\mid U)\right]\geq0.
 \label{eq:excessposterior}
\end{equation}
Substituting $\mathcal R(f_U)=\mathcal R^*(U)+\epsilon_U$ for $U=H$ and $U=(H,\hat Z)$ gives Equation~\ref{eq:riskaccounting}. Hence the refined readout improves exactly when $\epsilon_{H,\hat Z}-\epsilon_H<\Delta_{\mathrm{repr}}$. A positive representation opportunity can absorb some additional readout error; a sufficiently large increase in that error can instead offset it. Even when $\Delta_{\mathrm{repr}}=0$, an explicit schema can help a trained predictor if it makes the same Bayes decision easier to learn, reducing excess risk. When $\hat Z$ is a function of $H$, this latter route is the only one of the two available.

For the $0.8/0.2$ construction, an optimal coarse readout has risk $0.5$. A refined readout that ignores appraisal and always predicts class 1 has the same risk: its excess risk $0.3$ exactly offsets the representation gain. A refined readout that selects the wrong label in each state has risk $0.8$ and excess risk $0.6$; the correct refined Bayes readout has risk $0.2$ and excess risk zero. The same representation opportunity can therefore yield improvement, equality, or deterioration according to how the readout uses it.

\paragraph{Connection to the experimental comparisons.}
Generic semantics versus structured appraisal tests the choice of intermediate organization; cross-sample and within-emotion shuffling probe scene correspondence. These comparisons examine design consequences for learned behavior. Equation~\ref{eq:riskaccounting} compares readouts of fixed $H$ and $(H,\hat Z)$; separately trained common-backbone systems need not share the same learned $F$. Their aggregate accuracies therefore do not identify $\Delta_{\mathrm{repr}}$ or either excess-risk term. Likewise, appraisal replacement changes the arbitration weights and does not by itself estimate the cue interaction in Equation~\ref{eq:cueinteractionmain}. The theory specifies why joint conditioning and decision-relevant distinctions can matter; the experiments assess their benefit in the observed conflict settings.
\section{A guide to appraisal and scene interventions}
\label{app:interventionguide}
The contrasts distinguish whether appraisal belongs to the current scene, whether the decision uses it, and how prediction responds when scene meaning changes.

\subsection{Correspondence, connection, and content}
For a fixed input, write
\begin{equation}
 q_X(z)=D_{\psi}\bigl(H,\alpha(H,z)\bigr),\qquad H=F_{\theta}(X).
 \label{eq:appguideqx}
\end{equation}
The usual prediction evaluates this function at the predicted appraisal $\hat z=A_{\theta}(X)$. An alternative appraisal supplies a different reference for interpreting the same evidence. Table~\ref{tab:interventiontaxonomy} organizes the reported contrast families by the dependency they target.

\begin{table}[htbp]
\caption{Interpretive roles of the intervention families. Each row identifies a distinct question about the appraisal-conditioned decision pathway.}
\label{tab:interventiontaxonomy}
\small
\begin{tabularx}{\linewidth}{@{}>{\raggedright\arraybackslash}p{2.55cm}YY@{}}
\toprule
\VISTAHeaderRow
\VISTAFirstHeader{Contrast} & \VISTAHeader{Reference retained} & \VISTALastHeader{Dependency examined}\\
\midrule
Cross-sample shuffle & The current input and recognition task & Does another sample's appraisal supply the same useful interpretation as the scene's own appraisal?\\
Within-emotion shuffle & The current input and the donor's emotion identity & Does event-specific correspondence matter beyond the donor appraisal's association with the same class?\\
Appraisal with no decision connection & Auxiliary appraisal supervision & Does appraisal improve prediction when its forward connection to the decision is removed?\\
Generic semantic bottleneck & An intermediate semantic route to the same recognition task & Does organizing content as event appraisal contribute beyond the supplied semantic alternative?\\
Rationale / internal-state corruption & One of verbal presentation and internal appraisal is retained while the other is altered & Is the observed dependence stronger on displayed explanation text or on the internal appraisal information?\\
Relevant change / irrelevant paraphrase & The pairing identifies what event meaning should change or remain invariant & Does the response follow a relevant change while remaining stable to a reformulation of the same meaning?\\
\bottomrule
\end{tabularx}
\end{table}

Within-emotion shuffling preserves a donor's emotion identity while changing its relation to the current scene. Scenes that share an emotion can differ in their concerns, responsible agents, and expression conditions. Cross-sample shuffling relaxes even that class-level constraint. When strata use test gold labels, within-emotion shuffling is a gold-stratified offline mechanism diagnostic. Both replacements keep the checkpoint fixed and exchange the complete state. The connection ablation preserves auxiliary appraisal supervision while removing $z$ from arbitration; the ablated model is retrained. It asks whether predicting appraisal is sufficient without using it to select modality weights. Stop-gradient retains the forward information and changes its learning path, so it is a distinct operation.

\Needspace{8\baselineskip}
\subsection{Scene edits and representation replacement}
A scene edit can change both $H=F(X)$ and $A(X)$. Direct replacement of $z$ in $q_X(z)$ holds the evidence representation fixed. A training-term ablation changes the learned system. These intervention channels examine different dependencies even when they concern the same field.

Scene coherence also matters. If the goal changes while the outcome stays fixed, goal congruence may need to change with it. If the setting becomes public, the norm and expression-regulation interpretation may change together. A single-coordinate substitution can be useful as a representation probe, while a coherent event edit may involve several related fields. The paired scenes in Appendix~\ref{app:scenes} illustrate the latter relation between event facts and appraisal.

\subsection{Relevant response and label change are separate judgments}
A relevant event change need not force a new final label. An appeal can increase perceived control while leaving disappointment plausible; a formal congratulation can change a smile's diagnostic role without changing the strongest affective interpretation. Pair evaluation therefore distinguishes changes in appraisal, cue interpretation, and final prediction. Irrelevant paraphrases provide the complementary test: wording changes while the concern, outcome, and expression conditions are preserved.

\subsection{How appraisal changes the weighted readout}
\label{app:pathidentification}
The classification pathway makes the response to appraisal replacement explicit. Fix the checkpoint, every representation in $H$, and the availability mask $b$. Let $\mathcal M_b=\{m:b_m=1\}$ be the nonempty set of available modalities. For a reference appraisal $z$ and replacement $\tilde z$, write $\alpha=\alpha(H,z)$, $\tilde\alpha=\alpha(H,\tilde z)$, and $\Delta\alpha=\tilde\alpha-\alpha$. The branch logits $u_m=W_mh_m$ and residual $\rho=W_fh_{AVT}$ remain fixed. The two final logit vectors are
\begin{equation}
 v(z)=\rho+\sum_{m\in\mathcal M_b}\alpha_m u_m,\qquad
 v(\tilde z)=\rho+\sum_{m\in\mathcal M_b}\tilde\alpha_m u_m.
 \label{eq:fixedhreadout}
\end{equation}
Figure~\ref{fig:pathcontrasts} illustrates this comparison. Appraisal changes the allocation among the fixed branch logits; the shared residual cancels when taking their difference.

\begin{figure}[htbp]
\centering
\includegraphics[width=\linewidth]{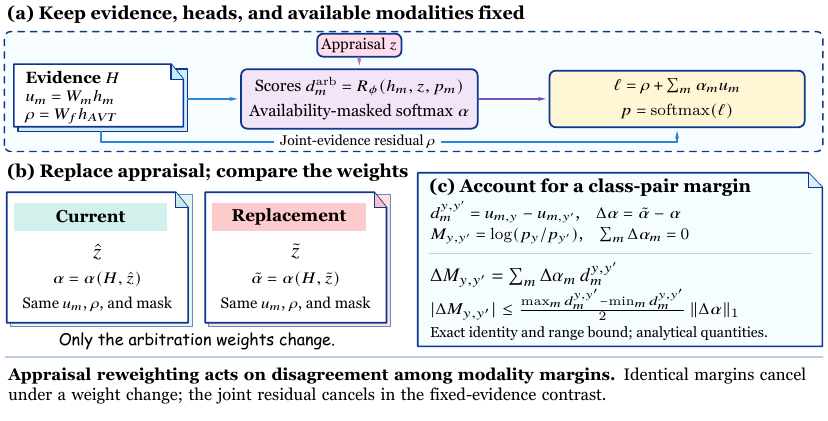}
\caption{\textbf{Appraisal replacement in the implemented weighted readout.} With the checkpoint, all evidence representations $H$, and available modalities fixed, reference and replacement appraisals produce different weights over the same unimodal logits. The joint-evidence residual is shared. For any class pair, the resulting logit-margin change is the weight change dotted with the branch margins; its magnitude is bounded by their range times half the $\ell_1$ weight change. The comparison describes how appraisal-dependent arbitration changes a class preference.}
\label{fig:pathcontrasts}
\end{figure}

\paragraph{Exact margin change.}
For two classes $y,y'$, define each branch's margin $d_m^{y,y'}=u_{m,y}-u_{m,y'}$ and the final margin $M_{y,y'}(z)=v_y(z)-v_{y'}(z)$. Since the final probabilities are a softmax, this also equals $\log[q_X(z)(y)/q_X(z)(y')]$. Subtraction in Equation~\ref{eq:fixedhreadout} gives
\begin{equation}
 M_{y,y'}(\tilde z)-M_{y,y'}(z)
 =\sum_{m\in\mathcal M_b}\Delta\alpha_m d_m^{y,y'},\qquad
 \sum_{m\in\mathcal M_b}\Delta\alpha_m=0.
 \label{eq:arbitrationmargin}
\end{equation}
The second identity follows because both weight vectors sum to one. Thus, moving weight from a branch with a smaller $y$--$y'$ margin to one with a larger margin increases the final preference for $y$ relative to $y'$. Opposing branch margins give appraisal a concrete means of changing which interpretation is favored. If all available branches have the same margin for this class pair, every redistribution leaves that margin unchanged.

\paragraph{A bound governed by branch disagreement.}
Let $d_{\max}=\max_{m\in\mathcal M_b}d_m^{y,y'}$ and $d_{\min}=\min_{m\in\mathcal M_b}d_m^{y,y'}$. The zero-sum identity yields the range bound
\begin{equation}
 \big|M_{y,y'}(\tilde z)-M_{y,y'}(z)\big|
 \leq\frac{d_{\max}-d_{\min}}{2}\,\|\Delta\alpha\|_1.
 \label{eq:arbitrationbound}
\end{equation}
To prove it, subtract $c=(d_{\max}+d_{\min})/2$ inside the sum in Equation~\ref{eq:arbitrationmargin}. This leaves the sum unchanged, and $|d_m^{y,y'}-c|\leq(d_{\max}-d_{\min})/2$. The triangle inequality completes the proof. Equivalently, the total positive and negative weight changes each have mass $\|\Delta\alpha\|_1/2$. The bound is attained when that mass is transferred between branches at the two extreme margins, whenever the corresponding weights are feasible.

These identities connect arbitration to conflicting evidence: appraisal can affect a class comparison when available branches differ in their support for that comparison. The sign of the weighted change determines which class is favored; the bound does not prescribe that sign or guarantee higher accuracy. The shared residual affects the reference margin and hence whether a given change crosses a decision boundary. In a multiclass task, identifying the new predicted label requires comparing the final logits across all classes.

The fixed-$H$ comparison isolates the implemented forward use of appraisal. Scene edits can also change branch logits and the residual, while retraining the connection ablation can change the whole predictor. Aggregate shuffle accuracies measure recognition after replacement; estimating Equation~\ref{eq:arbitrationmargin} per example additionally requires the paired internal weights and logits. External explanations are evaluated separately by holding $H$, $z$, and $\alpha$ fixed while rewriting the displayed text.

\subsection{Semantic organization and evaluative references}
\label{app:concernspecific}
The tested variants separate field names, represented content, scene correspondence, and decision access. Randomized field names retain 64.2\% conflict accuracy, close to the complete system's 64.5\%; a generic semantic bottleneck gives 63.2\%. The difference is associated with the organized appraisal content rather than the literal names of the fields. The goal and norm fields provide its evaluative reference, while expression regulation relates that reference to the observed display.

Goal congruence is defined relative to the person's concern. Consequently, coherent appraisal edits can involve more than one field: a changed goal can change congruence, and a changed audience can change social relevance and regulation. The constructed scenes in Appendix~\ref{app:scenes} make these dependencies explicit. The empirical intervention categories in Appendix~\ref{app:cfreconciliation} test directed responses to relevant changes, while unrelated reformulations test stability.

\subsection{Grounding the interface in observed context}
\label{app:grounding}
Context $C$ contains information available at recognition time. The goal and setting in Figure~\ref{fig:scene} are supplied premises; the illustration establishes a conditional interpretation of the same cue. In data, grounding asks whether utterances, visible events, or conversation history support the predicted field.

The field audit evaluates 400 candidates, recording coverage, sampling consistency, and adequate evidence separately. A field can be consistent across teacher samples yet poorly grounded; label-visible supervision makes this distinction concrete in Appendix~\ref{app:annotation}. Partial-field masking retains supported relations without treating every field as observed. Likewise, a wrong goal, an unsupported regulation inference, and a perceptual failure belong to different error categories. The grouped error analysis in Appendix~\ref{app:costerrorreading} connects these interpretive distinctions to the final recognition outcomes.
\clearpage
\section{Dataset roles and scoring conventions}
\label{app:data}
\begin{table}[H]
\caption{Public dataset context and the role of each evaluation. Dataset composition and the evaluation question are listed separately.}
\small
\begin{tabularx}{\linewidth}{@{}p{2.1cm}YY@{}}
\toprule
\VISTAHeaderRow
\VISTAFirstHeader{Dataset} & \VISTAHeader{Public context} & \VISTALastHeader{Role in this paper}\\
\midrule
CA-MER & 1,500 examples: 500 video-aligned, 500 audio-aligned, 500 consistent & Primary comparison of conflict and consistent performance\\
EmoMM & 4,000 base examples from Chinese CH-SIMS v2.0 and English CMU-MOSI; multimodal and unimodal sentiment annotations & Alignment, conflict, missingness, and their combination\\
CH-SIMS v2.0 & 4,402 labeled and 10,161 unlabeled Chinese clips, with multimodal and unimodal sentiment information & Binary accuracy (Acc2) across four conflict groups\\
THERADIA & 2,735 affect-annotated clips from interactions during cognitive exercises, with appraisal annotations & Frozen-representation appraisal probes and emotion-intensity regression\\
MELD & 13,708 utterances across 1,433 dialogues; official train/development/test sizes of 9,989/1,109/2,610 & Ordinary seven-class conversational emotion recognition\\
\bottomrule
\end{tabularx}
\end{table}

\paragraph{CA-MER and the evaluation interface.}
CA-MER supplies three annotation-defined groups, each with 500 examples: video-aligned conflict, audio-aligned conflict, and consistency \citep{han2025camer}. The present experiment uses a normalized output interface for all compared systems. Conflict Avg is the macro average of the two conflict accuracies; consistent accuracy, overall accuracy, and invalid-output rate are reported separately. MoSEAR and CHASE are evaluated through this same interface (Table~\ref{tab:externalcamer}). These local scores are distinguished from their original papers' open-vocabulary evaluation scores. The public benchmark's semantic set-matching convention is illustrated below as metric context, not used to reinterpret the present accuracy columns.

\paragraph{Nine-class output parsing.}
\label{app:camerparser}
The canonical vocabulary is \texttt{angry}, \texttt{happy}, \texttt{surprise}, \texttt{fear}, \texttt{sad}, \texttt{worry}, \texttt{neutral}, \texttt{doubt}, and \texttt{contempt}. The parser first strictly parses JSON and extracts a unique string-valued \texttt{final\_label}. It then applies Unicode NFKC normalization, trims leading and trailing whitespace, lowercases the string, applies the fixed synonyms in Table~\ref{tab:camersynonyms}, and checks the canonical vocabulary. A missing field, multiple answers, refusal, or unknown category is invalid and counts as an error in the main accuracy denominator. The output parser does not select an answer manually from explanatory prose.

\begin{table}[htbp]
\centering\small
\caption{\textbf{Fixed synonym normalization for the CA-MER interface.} Canonical labels pass through unchanged. The mapping is applied after JSON extraction and string normalization.}
\label{tab:camersynonyms}
\begin{tabularx}{\linewidth}{@{}YYYY@{}}
\toprule
\VISTAHeaderRow
\VISTAFirstHeader{Input}&\VISTAHeader{Canonical}&\VISTAHeader{Input}&\VISTALastHeader{Canonical}\\
\midrule
anger&angry&sadness&sad\\
happiness&happy&worried&worry\\
surprised&surprise&doubtful&doubt\\
fearful; afraid&fear&contemptuous&contempt\\
\bottomrule
\end{tabularx}
\end{table}

\subsection{Conflict definitions and error burden}
\label{app:conflictdefinition}
Equation~\ref{eq:conflictdefinition} separates the affect annotations used to define a conflict from the model whose recognition performance is evaluated. An eligible comparison concerns the same person, event, and temporal unit, with at least two available, affect-annotated modalities. Recognition-time context $C$ can explain a disagreement, but is not a fourth independently annotated affect channel. In particular, an error made by VISTA does not determine membership in the conflict subset.

\paragraph{Scalar sentiment: a specified scale and strict cutoff.}
For modality-only ratings $s_i^m\in[-1,1]$, the maximum pairwise gap has the equivalent range form
\begin{equation}
 \delta_i=\max_{m<n}|s_i^m-s_i^n|
 =\max_m s_i^m-\min_m s_i^m,\qquad
 \kappa_\tau(X_i)=\mathbf 1\{\delta_i>\tau\}.
 \label{eq:scalargap}
\end{equation}
The strong-conflict screening rule in \citet{wang2025diffemo} fixes $\tau=1$. It operates on the original rating scale supplied by separate modality annotations \citep{liu2022chsims}, and uses strict $>$: a gap exactly equal to 1 does not qualify. It should not be applied to arbitrary numeric category identifiers. The 661/4,402 count refers to the labeled CH-SIMS v2.0 corpus before severe cases are selected for the DiffEmo test. Table~\ref{tab:conflictevidence} keeps that corpus denominator separate from the test and error-audit denominators.

\Needspace{8\baselineskip}
\paragraph{Categorical affect: agreement classes and benchmark eligibility.}
Let $\equiv$ denote the benchmark's affect-label equivalence. A categorical pair has $d(q^m,q^n)=\mathbf 1\{q^m\not\equiv q^n\}$, with $\tau=0$. For our audit of the original CH-SIMS release \citep{yu2020chsims}, we use the fixed sentiment mapping in \citet{zhang2021crossmodal}:
\begin{equation}
 q(s)=\begin{cases}
 -1,&s<-0.1,\\
 0,&-0.1\leq s\leq0.1,\\
 +1,&s>0.1.
 \end{cases}
 \qquad
 K_i=\mathbf 1\{|\{q(s_i^t),q(s_i^a),q(s_i^v)\}|>1\}.
 \label{eq:chsimslabelaudit}
\end{equation}
All 2,281 released clips have complete modality ratings and unique video--clip identifier pairs; none is excluded. Summing $K_i$ gives 1,117/2,281, or 48.97\% (49.0\% in the introduction). This is a fresh descriptive count from public annotations, independent of VISTA predictions. It includes 455 clips with both positive and negative modalities, and 662 additional clips whose disagreement involves neutral and one non-neutral polarity. Each clip is counted once, even if multiple modality pairs disagree. The release contains movie, television-series, and variety-show clips \citep{yu2020chsims}; the frequency characterizes this corpus. Table~\ref{tab:chsimslabelaudit} gives pairwise and split-level counts. All counts follow Equation~\ref{eq:chsimslabelaudit}, comparing the original decimal annotations directly with the stated thresholds.

\begin{table}[htbp]
\centering\small
\caption{\textbf{Public-label audit of cross-modal disagreement on CH-SIMS.} All rows use Equation~\ref{eq:chsimslabelaudit}, with a restricted modality pair for the three pairwise rows. Pairwise rows overlap and are not summed. The final three rows partition the full corpus by the released split.}
\label{tab:chsimslabelaudit}
\begin{tabularx}{\linewidth}{@{}Yrrr@{}}
\toprule
\VISTAHeaderRow
\VISTAFirstHeader{Comparison / population} & \VISTAHeader{Disagreement} & \VISTAHeader{Eligible clips} & \VISTALastHeader{Rate (\%)}\\
\midrule
Any text--audio--video pair & 1,117 & 2,281 & 48.97\\
Text--audio & 800 & 2,281 & 35.07\\
Text--video & 945 & 2,281 & 41.43\\
Audio--video & 602 & 2,281 & 26.39\\
\midrule
Any pair: train & 676 & 1,368 & 49.42\\
Any pair: validation & 231 & 456 & 50.66\\
Any pair: test & 210 & 457 & 45.95\\
\bottomrule
\end{tabularx}
\end{table}

For CA-MER's audio, video, and multimodal construction labels $(q^a,q^v,q^{av})$, the evaluated groups are
\begin{equation}
\begin{aligned}
 \mathcal D_v&=\{i:q_i^v\equiv q_i^{av},\ q_i^a\not\equiv q_i^{av}\},\\
 \mathcal D_a&=\{i:q_i^a\equiv q_i^{av},\ q_i^v\not\equiv q_i^{av}\},\\
 \mathcal D_0&=\{i:q_i^a\equiv q_i^v\equiv q_i^{av}\}.
\end{aligned}
\label{eq:categoricalpartitions}
\end{equation}
Thus $\mathcal D_{\conf}=\mathcal D_v\cup\mathcal D_a$ and $\mathcal D_{\cons}=\mathcal D_0$. These are the public, annotator-checked partitions \citep{han2025camer}. Cases with all three labels different lie outside this three-group evaluation. The fixed category membership is independent of the evaluated model's confidence and output parser. ``Aligned'' describes agreement with the multimodal reference, not temporal synchronization.

\paragraph{Complementary definitions and ordered severity.}
EmoMM defines conflict when at least one available modality's annotated polarity differs from the multimodal reference; missingness is a separate controlled condition \citep{sun2026emomm}. Writing $\operatorname{pol}$ for the polarity mapping, its reference-based predicate is $\kappa_{\mathrm{ref}}(X)=\mathbf 1\{\exists m:\operatorname{pol}(q^m)\neq\operatorname{pol}(q^{\mathrm{joint}})\}$. It can flag a case even when unimodal polarities agree with each other but disagree with the joint reference. Polarity also differs from fine-grained affect categories: anger and sadness can be categorically distinct yet both negative. The CH-SIMS v2.0 Q1--Q4 results in Table~\ref{tab:chsims} use conflict strength calculated from unimodal labels, with group boundaries fixed on the training set. Equation~\ref{eq:scalargap} documents the literature's separate strong-conflict prevalence criterion; the label-visibility diagnostic additionally uses the reported threshold $C\geq2$ (Appendix~\ref{app:supervisionreading}). These named evaluation rules are not interchangeable.

\paragraph{Evidence across settings.}
In filtered Twitter image--text data, \citet{pan2024hybrid} report 42.5\% inconsistency in MVSA-Single and 26.0\% in MVSA-Multiple; their rule counts neutral versus positive/negative labels after opposite-polarity pairs have been removed. For user-generated video reviews, \citet{du2025sentiment} report 63.69\% agreement between textual and multimodal sentiment labels on UniC, hence 36.31\% disagreement. These complementary observations establish substantial cross-modal variation across data settings; the comparison target, annotation granularity, and corpus sampling determine which frequency is measured.

\begin{table}[htbp]
\centering
\caption{\textbf{Quantitative motivation with explicit denominators.} A public-label audit, published corpus analyses, and selected evaluation subsets answer different questions. Arithmetic gaps are calculated from the cited values.}
\label{tab:conflictevidence}
\small\setlength{\tabcolsep}{4pt}
\begin{tabularx}{\linewidth}{@{}>{\raggedright\arraybackslash}p{2.30cm}>{\raggedright\arraybackslash}p{2.20cm}Y@{}}
\toprule
\VISTAHeaderRow
\VISTAFirstHeader{Quantity} & \VISTAHeader{Reported evidence} & \VISTALastHeader{Population and interpretation}\\
\midrule
Three-modal disagreement & $48.97\%$\newline (1,117/2,281) & Our public-label audit of CH-SIMS: at least two unimodal polarity labels differ; Table~\ref{tab:chsimslabelaudit}.\\
\addlinespace[3pt]
Image--text inconsistency & $42.5\%$ / $26.0\%$ & Filtered MVSA-Single / Multiple; neutral versus positive or negative modality labels \citep{pan2024hybrid}.\\
\addlinespace[3pt]
Text--joint disagreement & $36.31\%$ & UniC video reviews: complement of 63.69\% text-only versus multimodal label agreement \citep{du2025sentiment}.\\
\addlinespace[3pt]
Strong-conflict prevalence & $15.0\%$\newline (661/4,402) & Labeled CH-SIMS v2.0 clips screened at $\delta>1$; corpus frequency \citep{wang2025diffemo}.\\
\addlinespace[3pt]
Conflict difficulty & $23.53$ pp & MulT binary accuracy: 89.13\% aligned, 65.60\% conflict; selected DiffEmo tests, 173 examples per group \citep{wang2025diffemo}.\\
\addlinespace[3pt]
Primary benchmark mix & 1,000/1,500\newline conflict & CA-MER's equal allocation across two conflict directions and consistency; evaluation design \citep{han2025camer}.\\
\bottomrule
\end{tabularx}
\end{table}

\paragraph{Three distinct error questions.}
For a common evaluation population, let $E_i=\mathbf 1\{\hat y_i\neq y_i\}$ denote a specified single-label prediction error and $K_i=\kappa_\tau(X_i)$. Conflict prevalence is $\pi=P(K=1)$; conditional error rates are $e_1=P(E=1\mid K=1)$ and $e_0=P(E=1\mid K=0)$. Bayes' rule gives the fraction of errors that occur on conflict inputs:
\begin{equation}
 \rho=P(K=1\mid E=1)
 =\frac{\pi e_1}{\pi e_1+(1-\pi)e_0},
 \qquad P(E=1)>0.
 \label{eq:conflicterrorburden}
\end{equation}
This is a concentration of errors within a common sampling frame. Causal attribution additionally asks which errors would change under a specified conflict-resolution intervention. An error-category audit supplies a different descriptive quantity. For an error group $G$ and category $k$, write
\begin{equation}
 \eta(k\mid G)=
 \frac{\#\{\text{reviewed errors in }G\text{ assigned category }k\}}
 {\#\{\text{reviewed errors in }G\}}.
 \label{eq:arbitrationauditshare}
\end{equation}
The final error analysis instead conditions on three paired-model groups: errors made only by Emotion-SFT, errors made only by VISTA, and errors shared by both models. Table~\ref{tab:errors} preserves these separate denominators. A category share within one of these groups is neither a corpus conflict prevalence nor the fraction of all recognition errors caused by conflict.

Equation~\ref{eq:conflicterrorburden} requires prevalence and both conditional rates from the same population. The corpus screening frequency and severity-selected DiffEmo test rates therefore remain separate measurements.

\paragraph{Threshold sensitivity and a reproducible audit.}
For a fixed eligible sample and a common rating scale,
\begin{equation}
 \widehat\pi(\tau)=\frac1N\sum_{i=1}^{N}\mathbf 1\{\delta_i>\tau\},
 \quad \tau_1<\tau_2\ \Longrightarrow\
 \widehat\pi(\tau_2)\leq\widehat\pi(\tau_1).
 \label{eq:thresholdprevalence}
\end{equation}
A threshold changes which severity levels are counted; monotonic prevalence does not imply monotonic model error. A reproducible analysis stores sample IDs, independently obtained modality labels, missingness, rating normalization, the fixed comparison rule, prediction and reference labels, and any audit category. It reports the eligible count, conflict count, and error count together. Threshold selection uses annotation semantics or validation data and is fixed before examining test performance; a sensitivity curve can then report how prevalence and recognition change across cutoffs. These definitions specify the records needed to reproduce prevalence and error-burden estimates.

\subsection{Public benchmark metrics and the normalized interface}
Let $\mathcal P_i$ and $\mathcal G_i$ denote predicted and reference emotion sets after semantic matching in the public CA-MER protocol. With $t_i=|\mathcal P_i\cap\mathcal G_i|$, $p_i=|\mathcal P_i|$, and $g_i=|\mathcal G_i|$, its nonempty-set metrics are
\begin{equation}
 a_i=\frac{t_i}{p_i},\qquad r_i=\frac{t_i}{g_i},\qquad
 b_i=\frac{a_i+r_i}{2}.
 \label{eq:setmetrics}
\end{equation}
These quantities explain why a published open-vocabulary score and a local normalized-interface accuracy can differ. In particular, the public Acc is precision-like for multiple predicted labels. Figure~\ref{fig:label_sets} gives an exact constructed illustration. It contains no trained-model evaluation results.

\begin{figure}[htbp]
\centering
\includegraphics[width=\linewidth]{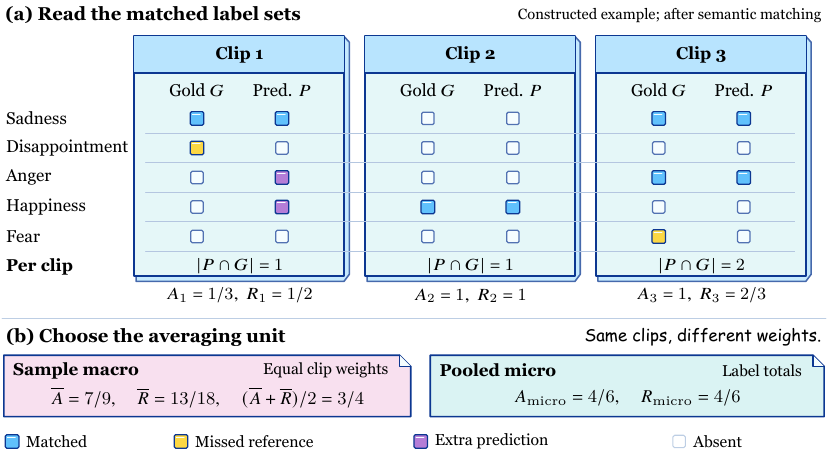}
\caption{\textbf{A constructed illustration of the public set-matching metric.} After semantic matching, three illustrative clips give equal-clip Acc $7/9$, recall $13/18$, and their arithmetic mean $3/4$. Pooling memberships instead gives $2/3$ for both metrics. This public-protocol example is separate from the normalized-interface experiments reported in this paper.}
\label{fig:label_sets}
\end{figure}

Equal-clip averages and pooled memberships assign different weights:
\begin{equation}
 a_{\mathrm{pool}}=\frac{\sum_i t_i}{\sum_i p_i}
 =\sum_i\frac{p_i}{\sum_jp_j}a_i,\qquad
 r_{\mathrm{pool}}=\frac{\sum_i t_i}{\sum_i g_i}
 =\sum_i\frac{g_i}{\sum_jg_j}r_i.
 \label{eq:setpooling}
\end{equation}
Accordingly, the present work compares MoSEAR, CHASE, and VISTA only under their shared normalized interface; their original published scores remain distinct comparison contexts.

\paragraph{Conflict and overall subset weights.}
CA-MER's two conflict directions receive equal weight in Conflict Avg. Across the three equally sized groups,
\begin{equation}
 A_{\conf}=\frac{A_v+A_a}{2},\qquad
 A_{\mathrm{all}}=\frac{A_v+A_a+A_{\cons}}{3}
 =\frac23 A_{\conf}+\frac13 A_{\cons}.
 \label{eq:camerweights}
\end{equation}
For VISTA, the displayed subgroup means give $(61.0+68.0+74.2)/3\approx67.7\%$, matching the reported overall accuracy after rounding. The conflict gain and conflict specificity ask complementary questions: how much recognition improves under conflict, and how much of that improvement exceeds the consistent-condition gain.

\subsection{Complementary evaluation settings}
\paragraph{EmoMM: evaluation without benchmark-specific adaptation.}
EmoMM draws 2,000 Chinese and 2,000 English base examples and constructs conflict and missing-modality conditions \citep{sun2026emomm}. Our evaluation follows the source-level test split with 800 sources; a source can generate multiple condition-specific inputs, so this source count is distinct from the input count for any condition. In our evaluation, all five core models transfer without EmoMM adaptation. Conflict-plus-missing inputs overlap with the missing condition, so the four-condition macro average is an equally weighted condition summary rather than full-sample accuracy. CHASE is listed separately because its original router used EmoMM supervision. EmoMM and the common training pool share CH-SIMS v2 source material. Freezing specifies the absence of EmoMM-specific updates; sample disjointness requires a separate source-identity comparison.

\paragraph{CH-SIMS v2.0: task adaptation and conflict severity.}
Trained models receive CH-SIMS-specific adaptation; Base is evaluated zero-shot. On complete text, audio, and video labels $y_i^T,y_i^A,y_i^V\in[-1,1]$, conflict intensity is
\begin{equation}
 \begin{aligned}
 C_i&=|y_i^T-y_i^A|+|y_i^T-y_i^V|+|y_i^A-y_i^V|\\
 &=2\left(\max_{m\in\{T,A,V\}}y_i^m-\min_{m\in\{T,A,V\}}y_i^m\right)\in[0,4].
 \end{aligned}
 \label{eq:simsconflictintensity}
\end{equation}
Thus $C_i\geq2$ means a unimodal range of at least $1$. It measures label disagreement on the original scale; a missing modality label is not replaced by zero.

Our protocol specifies train/development/test sizes of $2{,}722/647/1{,}034$.\footnote{These supplied split sizes sum to $4{,}403$, whereas the public corpus description gives $4{,}402$ labeled clips. We retain the protocol split counts and public corpus size separately.} The training split defines Q1--Q4 by lexicographically sorting the pair $(C_i,h_i)$, where $h_i$ is a stable-ID hash used to break ties. With 2,722 complete training examples, the cut keys are at ranks 680, 1,361, and 2,041. For a test key $q=(C,h)$ and ordered cut keys $b_1,b_2,b_3$, the groups are $q\leq b_1$, $b_1<q\leq b_2$, $b_2<q\leq b_3$, and $q>b_3$. Training-defined keys are applied unchanged to test inputs; rounding the conflict values alone would discard the tie rule.

Acc2 includes zero-valued reference labels and splits sentiment into $\leq0$ and $>0$. MAE uses the $[-1,1]$ target scale. An invalid generated output counts as incorrect for Acc2 and receives an absolute-error penalty of $2$ for MAE. Table~\ref{tab:chsims} reports Acc2 for all four groups, overall Acc2, and Q4 MAE. Unlabeled public clips are not added to a labeled evaluation denominator.

\paragraph{THERADIA: representation and downstream use.}
The THERADIA WoZ split contains 1,110 training, 851 development, and 774 test video entries; field-valid sample counts depend on the annotation required by each task. For our appraisal probe, each backbone is frozen and receives the same four-dimensional linear readout. Novelty, intrinsic pleasantness, goal conduciveness, and coping are the available human appraisal dimensions \citep{fournier2024theradia}. Novelty is not identified with VISTA's Expectation field, and these labels do not cover the entire seven-field schema. A separate regression predicts the intensities of ten named emotions to evaluate whether appraisal improves downstream prediction (Appendix~\ref{app:learnabilityuse}).

\paragraph{MELD: ordinary recognition.}
In our evaluation, trained models receive MELD-specific adaptation; Base uses zero-shot evaluation. Weighted-F1, macro-F1, and accuracy assess seven-class conversational emotion recognition \citep{poria2019meld}. Table~\ref{tab:meld} and Appendix~\ref{app:ordinary} report the aggregate and available class-specific results.
\section{Conflict specificity and directional improvement}
\label{app:specificityalgebra}\label{app:specificityreading}\label{app:empiricalreading}
Let $A_{\conf}(M)$ and $A_{\cons}(M)$ denote the conflict and consistent scores of model $M$ under a shared evaluation. Relative to a comparator $B$, specificity has two equivalent readings:
\begin{align}
 S(\text{\method{}},B)
 &=\big[A_{\conf}(\text{\method{}})-A_{\conf}(B)\big]
   -\big[A_{\cons}(\text{\method{}})-A_{\cons}(B)\big] \nonumber\\
 &=\big[A_{\cons}(B)-A_{\conf}(B)\big]
   -\big[A_{\cons}(\text{\method{}})-A_{\conf}(\text{\method{}})\big].
 \label{eq:specificitydeficit}
\end{align}
It measures the excess gain on conflict inputs, or equivalently the contraction of the conflict deficit. Emotion-SFT has a $73.0-60.0=13.0$ pp deficit and VISTA a $74.2-64.5=9.7$ pp deficit, yielding $S=3.3$ pp. Against Generic-CoT-SFT and the learned gate, VISTA's conflict gains remain larger than its consistent gains (Table~\ref{tab:specificityfull}).

\begin{table}[htbp]
\centering\small
\caption{\textbf{CA-MER conflict specificity (pp).} The same comparator is used for both gains. Contrasts are calculated before rounding the displayed accuracies.}
\label{tab:specificityfull}
\begin{tabular}{@{}lrrr@{}}
\toprule
\VISTAHeaderRow
\VISTAFirstHeader{Comparator} & \VISTAHeader{$\Delta_{\conf}$} & \VISTAHeader{$\Delta_{\cons}$} & \VISTALastHeader{$S$}\\
\midrule
Emotion-SFT &4.5&1.2&3.3\\
Generic-CoT-SFT &3.0&0.2&2.8\\
Modality-Gate-SFT &2.5&0.2&2.3\\
\bottomrule
\end{tabular}
\end{table}

\subsection{Improvement in both conflict directions}
Write $\Delta_v=A_v(\method)-A_v(B)$ and $\Delta_a=A_a(\method)-A_a(B)$. Their mean and half-difference are
\begin{align}
 \mu_B&=\frac{\Delta_v+\Delta_a}{2},\qquad
 d_B=\frac{\Delta_v-\Delta_a}{2},\nonumber\\
 (\Delta_v,\Delta_a)&=(\mu_B+d_B,\mu_B-d_B).
 \label{eq:directiondecomposition}
\end{align}
Both directions improve precisely when $\mu_B>|d_B|$. Subtracting the same consistent gain gives
\begin{equation}
 S_v=\Delta_v-\Delta_{\cons},\qquad
 S_a=\Delta_a-\Delta_{\cons},\qquad
 \frac{S_v+S_a}{2}=S.
 \label{eq:directionalspecificity}
\end{equation}
Table~\ref{tab:directionalcontrasts} shows positive gains and positive directional specificity against all three trained controls. Thus the average conflict improvement is shared by both annotation-defined conflict directions.

\begin{table}[htbp]
\centering\small
\caption{\textbf{Directional control contrasts (pp).} Calculated from the displayed subgroup means in Table~\ref{tab:camercomplete}. The final column is the reduction in the audio-minus-video disparity.}
\label{tab:directionalcontrasts}
\begin{tabular}{@{}lrrrrrr@{}}
\toprule
\VISTAHeaderRow
\VISTAFirstHeader{Comparator} & \VISTAHeader{$\Delta_v$} & \VISTAHeader{$\Delta_a$} & \VISTAHeader{$\Delta_{\cons}$} & \VISTAHeader{$S_v$} & \VISTAHeader{$S_a$} & \VISTALastHeader{$2d_B$}\\
\midrule
Emotion-SFT &6.0&3.0&1.2&4.8&1.8&3.0\\
Generic-CoT-SFT &4.0&2.0&0.2&3.8&1.8&2.0\\
Modality-Gate-SFT &3.0&2.0&0.2&2.8&1.8&1.0\\
\bottomrule
\end{tabular}
\end{table}

\paragraph{Specificity under representation interventions.}
The full system has $S=3.3$ pp against Emotion-SFT. Cross-sample shuffling and cutting appraisal's decision connection each give $S=1.5$ pp; same-emotion shuffling gives $S=2.7$ pp. The paired contrast identity
\begin{equation}
 d_{\conf}(V)-d_{\cons}(V)=S(\method,B)-S(V,B),
 \quad d_u(V)=A_u(\method)-A_u(V),
 \label{eq:attenuationreading}
\end{equation}
connects the ablations to the conflict-specific gain. For cross-sample shuffling, the conflict and consistent decreases are $2.5$ and $0.7$ pp, giving a $1.8$ pp loss of specificity. For same-emotion shuffling they are $0.8$ and $0.2$ pp, giving $0.6$ pp. These descriptive contrasts assess how much the observed advantage depends on the supplied appraisal state, without interpreting a share of accuracy gain as a fraction of causally attributable errors.

Specificity is reported together with absolute accuracies because a positive difference in gains alone need not imply improvement. The present comparisons have both. These aggregate scores and contrasts are descriptive; they do not establish statistical significance or formal non-inferiority.
\FloatBarrier
\stopcontents[vistaconcepts]
\resumecontents[vistaevidence]
\section{Complete recognition results and mechanism controls}
\label{app:additional}
\subsection{CA-MER: all subsets and output validity}
\label{app:coverage}
Table~\ref{tab:camercomplete} supplies every subset, aggregate, and output-validity measurement. Conflict Avg gives equal weight to the visual and audio conflict directions. The main-text comparison emphasizes conflict specificity; the full table also reports overall accuracy and output validity for every model.

\begin{table}[htbp]
\centering\small\setlength{\tabcolsep}{4pt}
\caption{\textbf{Complete CA-MER results (\%).} All five core models use Qwen2.5-Omni-7B and the same normalized output interface. Conflict Avg is $(A_v+A_a)/2$.}
\label{tab:camercomplete}\label{tab:coverage}
\begin{tabular}{@{}lrrrrrr@{}}
\toprule
\VISTAHeaderRow
\VISTAFirstHeader{Method}&\VISTAHeader{Visual}&\VISTAHeader{Audio}&\VISTAHeader{Conflict Avg}&\VISTAHeader{Consistent}&\VISTAHeader{Overall}&\VISTALastHeader{Invalid}\\
\midrule
Base&49.0&60.0&54.5&68.0&59.0&0.5\\
Emotion-SFT&55.0&65.0&60.0&73.0&64.3&0.2\\
Generic-CoT-SFT&57.0&66.0&61.5&74.0&65.7&0.4\\
Modality-Gate-SFT&58.0&66.0&62.0&74.0&66.0&0.4\\
\VISTAFocusRow
\VISTAFirstFocus{\method{}}&61.0&68.0&64.5&74.2&67.7&\VISTALastFocus 0.7\\
\bottomrule
\end{tabular}
\end{table}

\begin{table}[htbp]
\centering\small
\caption{\textbf{External methods under the normalized CA-MER interface.} These reevaluations are separate from the methods' originally published open-vocabulary scores. Differences are in pp.}
\label{tab:externalcamer}
\begin{tabular}{@{}lrr@{}}
\toprule
\VISTAHeaderRow
\VISTAFirstHeader{Method}&\VISTAHeader{Conflict Avg (\%)}&\VISTALastHeader{VISTA gain}\\
\midrule
MoSEAR&61.5&3.0\\
CHASE&61.0&3.5\\
\VISTAFocusRow
\VISTAFirstFocus{\method{}}&64.5&\VISTALastFocus ---\\
\bottomrule
\end{tabular}
\end{table}

\subsection{EmoMM: conflict, missingness, and their intersection}
VISTA transfers to EmoMM without additional adaptation. Relative to Generic-CoT-SFT, its gains are $4.0$ pp under conflict and $4.9$ pp under conflict plus missingness. The consistent-to-conflict decrease is $4.3$ pp for VISTA, compared with $7.5$, $7.5$, and $6.6$ pp for Emotion-SFT, Generic-CoT-SFT, and Modality-Gate-SFT. Table~\ref{tab:emomm} also includes CHASE's protocol reevaluation. Its original router used EmoMM supervision; it attains higher accuracy than VISTA under missingness ($49.3$ versus $46.8$) and combined conflict and missingness ($44.5$ versus $43.5$). This comparison distinguishes conflict transfer from the additional challenge of recovering absent evidence.

\begin{table}[htbp]
\centering\small\setlength{\tabcolsep}{4.2pt}
\caption{\textbf{Complete EmoMM accuracy (\%).} The five core models use frozen transfer. CHASE is a separate protocol reevaluation with prior EmoMM supervision. Macro Avg weights conditions equally; overlapping missingness conditions prevent interpreting it as full-sample overall accuracy.}
\label{tab:emomm}
\begin{tabular}{@{}lrrrrr@{}}
\toprule
\VISTAHeaderRow
\VISTAFirstHeader{Method}&\VISTAHeader{Consistent}&\VISTAHeader{Conflict}&\VISTAHeader{Missing}&\VISTAHeader{Conflict + missing}&\VISTALastHeader{Macro Avg}\\
\midrule
Base&52.1&46.5&43.3&38.0&45.0\\
Emotion-SFT&55.0&47.5&44.0&38.0&46.1\\
Generic-CoT-SFT&55.5&48.0&44.2&38.6&46.6\\
Modality-Gate-SFT&55.6&49.0&45.1&40.0&47.4\\
\VISTAFocusRow
\VISTAFirstFocus{\method{}}&56.3&52.0&46.8&43.5&\VISTALastFocus 49.7\\
\midrule
CHASE (reevaluated)&53.2&51.3&49.3&44.5&49.6\\
\bottomrule
\end{tabular}
\end{table}

\subsection{CH-SIMS v2.0: increasing conflict intensity}
The adapted VISTA model gains $0.2$, $0.7$, $2.0$, and $4.5$ pp over Emotion-SFT from Q1 to Q4. The highest-conflict group's MAE falls from $0.365$ to $0.315$. Table~\ref{tab:chsims} includes zero-shot Base, the complete quartile profile, overall binary accuracy, and Q4 regression error.

\begin{table}[htbp]
\centering\small\setlength{\tabcolsep}{4.8pt}
\caption{\textbf{CH-SIMS v2.0 by conflict strength.} Q1--Q4 use increasing conflict intensity with boundaries set on training data. Acc2 is binary accuracy in percent; Q4 MAE retains its original scale. Base is zero-shot; the four trained methods receive task adaptation.}
\label{tab:chsims}
\begin{tabular}{@{}lrrrrrr@{}}
\toprule
\VISTAHeaderRow
\VISTAFirstHeader{Method}&\VISTAHeader{Q1}&\VISTAHeader{Q2}&\VISTAHeader{Q3}&\VISTAHeader{Q4}&\VISTAHeader{Overall Acc2}&\VISTALastHeader{Q4 MAE $\downarrow$}\\
\midrule
Base&87.5&82.5&77.5&70.0&79.4&0.460\\
Emotion-SFT&90.0&86.5&82.5&77.0&84.0&0.365\\
Generic-CoT-SFT&89.8&86.7&82.8&77.8&84.3&0.355\\
Modality-Gate-SFT&89.7&86.8&83.5&79.2&84.8&0.337\\
\VISTAFocusRow
\VISTAFirstFocus{\method{}}&90.2&87.2&84.5&81.5&85.9&\VISTALastFocus 0.315\\
\bottomrule
\end{tabular}
\end{table}

\paragraph{A common trend under all three controls.}
Define $g_q(B)=A_q(\method)-A_q(B)$. The endpoint contrast and the descriptive least-squares slope across ordered group indices are
\begin{equation}
 E_B=g_4(B)-g_1(B),\qquad
 \beta_B=\frac{\sum_{q=1}^4(q-2.5)g_q(B)}{5}.
 \label{eq:quartilegaintrend}
\end{equation}
All three comparisons have their largest gain in Q4 and a positive endpoint contrast (Table~\ref{tab:quartilecontrasts}); the gain over Emotion-SFT increases at every group. Group indices express severity order, not equal measured increments of conflict; these slopes summarize the four aggregate gains rather than an individual-level causal response.

\begin{table}[htbp]
\centering\small
\caption{\textbf{CH-SIMS gains calculated from Table~\ref{tab:chsims}.} Gains and endpoint contrasts are in pp; slopes are in pp per group index.}
\label{tab:quartilecontrasts}
\begin{tabular}{@{}lrrrrrr@{}}
\toprule
\VISTAHeaderRow
\VISTAFirstHeader{Comparator}&\VISTAHeader{$g_1$}&\VISTAHeader{$g_2$}&\VISTAHeader{$g_3$}&\VISTAHeader{$g_4$}&\VISTAHeader{$E_B$}&\VISTALastHeader{$\beta_B$}\\
\midrule
Emotion-SFT&0.2&0.7&2.0&4.5&4.3&1.42\\
Generic-CoT-SFT&0.4&0.5&1.7&3.7&3.3&1.11\\
Modality-Gate-SFT&0.5&0.4&1.0&2.3&1.8&0.60\\
\bottomrule
\end{tabular}
\end{table}

\subsection{The complete appraisal and training ablations}
Table~\ref{tab:fullablations} separates the content of appraisal, its connection to the decision, and its supervision. Cross-sample shuffling reduces conflict accuracy by $2.5$ pp, whereas shuffling within the same emotion class reduces it by $0.8$ pp. Randomizing field names changes the score by $0.3$ pp; replacing the representation with a generic semantic bottleneck changes it by $1.3$ pp. The combination favors informative scene-specific content and decision access over the names of the fields alone.

\begin{table}[htbp]
\centering\small\setlength{\tabcolsep}{4.2pt}
\caption{\textbf{Complete CA-MER mechanism controls.} Accuracy is in percent; changes and specificity are in pp. $S$ is always relative to Emotion-SFT. ``Auxiliary only'' retains appraisal supervision but cuts its decision connection. The $\lambda_{\mathrm{cf}}=0$ row repeats the full configuration.}
\label{tab:fullablations}
\begin{tabularx}{\linewidth}{@{}Yrrrr@{}}
\toprule
\VISTAHeaderRow
\VISTAFirstHeader{Condition}&\VISTAHeader{Conflict}&\VISTAHeader{Consistent}&\VISTAHeader{$\Delta_{\conf}$}&\VISTALastHeader{$S$}\\
\midrule
\VISTAFocusRow
\VISTAFirstFocus{Full \method{}}&64.5&74.2&0.0&\VISTALastFocus 3.3\\
Within-emotion appraisal shuffle&63.7&74.0&$-0.8$&2.7\\
Cross-sample appraisal shuffle&62.0&73.5&$-2.5$&1.5\\
Field-name randomization&64.2&74.2&$-0.3$&3.0\\
Auxiliary only; decision connection cut&62.5&74.0&$-2.0$&1.5\\
No expression regulation $r$&63.6&74.0&$-0.9$&2.6\\
Generic semantic bottleneck&63.2&74.1&$-1.3$&2.1\\
External explanation changed; internal $z$ fixed&64.5&74.2&0.0&3.3\\
Internal $z$ corrupted; external explanation retained&62.0&73.5&$-2.5$&1.5\\
No appraisal supervision ($\lambda_z=0$)&63.2&74.0&$-1.3$&2.2\\
No conflict supervision ($\lambda_c=0$)&63.5&74.1&$-1.0$&2.4\\
No reliability supervision ($\lambda_a=0$)&63.0&74.0&$-1.5$&2.0\\
$\lambda_{\mathrm{cf}}=0$ (full configuration)&64.5&74.2&0.0&3.3\\
\bottomrule
\end{tabularx}
\end{table}

Changing the external explanation while holding the internal state fixed leaves the reported accuracies unchanged. Corrupting the internal appraisal while retaining the explanation reproduces the cross-sample-shuffle scores. VISTA uses $\lambda_{\mathrm{cf}}=0$: its counterfactual evaluation is therefore a behavioral test, not an ablation of an active counterfactual training loss. Removing appraisal, conflict, and reliability supervision gives conflict reductions of $1.3$, $1.0$, and $1.5$ pp, respectively.

\paragraph{Native modality weights.}
VISTA's visual weight exceeds its audio weight on visual-aligned conflict, and the ordering reverses on audio-aligned conflict (Table~\ref{tab:nativeweights}). The gate's corresponding dominant weights are $0.37$ and $0.43$. These subset means describe conditional weighting behavior. They are distinct from the shared masking diagnostic $\alpha_{\mathrm{diag}}$ used for the interventions in Appendix~\ref{app:counterfactualreading}.

\begin{table}[htbp]
\centering\small
\caption{\textbf{Reported native modality weights $\alpha$.} Dashes mark gate weights for which no separate value is available; these entries are not inferred from normalization. These are weights on their original scale.}
\label{tab:nativeweights}
\begin{tabular}{@{}llrr@{}}
\toprule
\VISTAHeaderRow
\VISTAFirstHeader{Method}&\VISTAHeader{Conflict subset}&\VISTAHeader{Visual weight}&\VISTALastHeader{Audio weight}\\
\midrule
Modality-Gate-SFT&Visual-aligned&0.37&---\\
Modality-Gate-SFT&Audio-aligned&---&0.43\\
VISTA&Visual-aligned&0.42&0.30\\
VISTA&Audio-aligned&0.25&0.47\\
\bottomrule
\end{tabular}
\end{table}

\subsection{Frozen-backbone prompts and explanation quality}
All four prompting conditions hold the backbone frozen. Appraisal prompting improves Conflict Avg from $54.5$ to $56.0\%$ relative to direct prediction, while generic conflict CoT reaches $56.4\%$. Appraisal prompting also produces the longest outputs and the largest invalid-output rate (Table~\ref{tab:frozenmain}). Thus the trained system's improvement is not explained by a superior appraisal prompt alone.

\begin{table}[!ht]
\centering\small
\caption{\textbf{Complete frozen-model prompt comparison.} Recognition and invalid-output rates are percentages; output length is the mean token count.}
\label{tab:frozenmain}
\begin{tabular}{@{}lrrrr@{}}
\toprule
\VISTAHeaderRow
\VISTAFirstHeader{Prompt}&\VISTAHeader{Conflict Avg}&\VISTAHeader{Consistent}&\VISTAHeader{Invalid}&\VISTALastHeader{Output tokens}\\
\midrule
Direct&54.5&68.0&0.5&8\\
Modality decomposition&55.8&69.0&0.8&210\\
Generic conflict CoT&56.4&69.0&1.0&290\\
Appraisal prompt&56.0&68.5&1.8&440\\
\bottomrule
\end{tabular}
\end{table}

\begin{figure}[htbp]
\centering
\includegraphics[width=\linewidth]{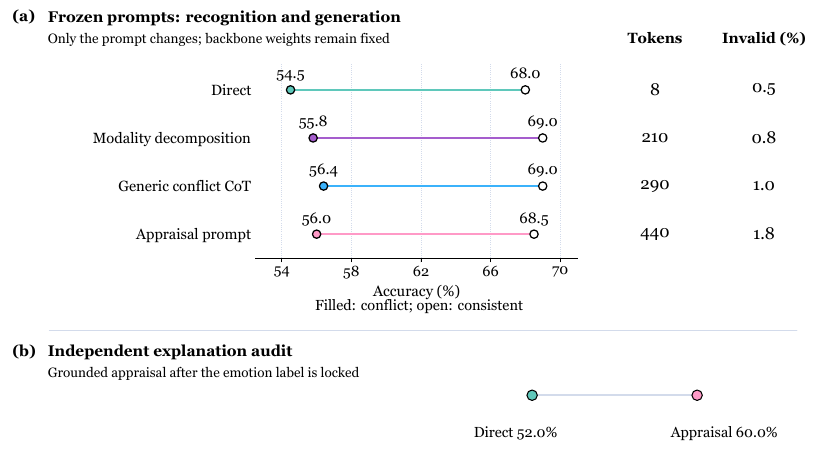}
\caption{\textbf{What prompting supplies before training.} (a) Conflict and consistent accuracy of four frozen-backbone prompts, with output length and invalid-output rates. (b) Grounded-appraisal rates of $52.0\%$ for Direct and $60.0\%$ for Appraisal prompt in a separate explanation audit. That audit locks the predicted label before eliciting the second-pass explanation.}
\label{fig:score_geometry}
\end{figure}

The independent second-pass audit evaluates explanations after the emotion label has been locked. Grounded-appraisal rates are $52.0\%$ for Direct and $60.0\%$ for Appraisal prompt. This measures explanation quality under the audit protocol, separately from the first-pass recognition scores. More grounded explanations do not by themselves imply that the corresponding appraisal prompt has higher conflict accuracy than every alternative.
\FloatBarrier
\section{Counterfactual response and irrelevant rewrites}
\label{app:counterfactualreading}
\subsection{Two paired tests with different success criteria}
\label{app:cfreconciliation}
The evaluation contains 150 valid intervention pairs and a separate set of 150 irrelevant-rewrite pairs. A valid intervention changes information relevant to the affective interpretation; its direction score asks whether the target-category probability moves in the annotated direction. An irrelevant rewrite tests whether the predicted label remains unchanged. These criteria distinguish sensitivity to a meaningful change from stability under a semantically irrelevant one. The paired-input test changes $X$ and can therefore change both evidence $H$ and appraisal $z$; a state-only replacement with fixed $X$ is a separate intervention.

For a fixed set of paired predictions with target category $y_i^\star$, paired inputs $(X_i,X'_i)$, and annotated direction $d_i\in\{-1,+1\}$, the direction criterion is
\begin{equation}
 \mathrm{Dir}=\frac{1}{N}\sum_{i=1}^{N}
 \mathbf 1\!\left\{d_i\left[p(y_i^\star\mid X'_i)-p(y_i^\star\mid X_i)\right]>0\right\}.
 \label{eq:cfdirscore}
\end{equation}
The target category and direction are fixed before either prediction is scored. A zero directional change is a failure, and invalid model outputs remain failures in the fixed direction-score denominator. A label can remain the most likely class while its probability moves in the appropriate direction. The valid-intervention label-change rate,
\begin{equation}
 \mathrm{Change}=\frac1N\sum_i\mathbf 1\{\hat y(X'_i)\neq\hat y(X_i)\},
 \label{eq:cflabelchange}
\end{equation}
is therefore a separate measurement, not a substitute for direction compliance. Irrelevant-rewrite stability uses the complementary equality criterion on its own pair set. The signed probability difference averages $d_i[p(y_i^\star\mid X_i')-p(y_i^\star\mid X_i)]$. The signed log-odds difference uses the same target and direction, with probabilities clipped to $[10^{-6},1-10^{-6}]$ for numerical stability. These continuous summaries use pairs with defined probabilities; their valid count is distinct from the fixed binary-score denominator.

Tables~\ref{tab:counterfactual} and~\ref{tab:counterfactualfields} retain the source-reported aggregate precision. These aggregate percentages are not converted into integer per-run counts or reconstructed uncertainty estimates.

\begin{table}[htbp]
\centering\small\setlength{\tabcolsep}{4pt}
\caption{\textbf{Complete paired intervention results.} Direction, label change, and rewrite stability are percentages. Probability and log-odds differences retain their original scales. Valid interventions and irrelevant rewrites each contain 150 pairs.}
\label{tab:counterfactual}
\begin{tabular}{@{}lrrrrr@{}}
\toprule
\VISTAHeaderRow
\VISTAFirstHeader{Method}&\VISTAHeader{Direction}&\VISTAHeader{Label change}&\VISTAHeader{$\Delta p$}&\VISTAHeader{$\Delta$ log-odds}&\VISTALastHeader{Rewrite stability}\\
\midrule
Emotion-SFT&62.0&24.0&0.050&0.201&91.0\\
Generic-CoT-SFT&66.0&28.0&0.070&0.281&90.0\\
Modality-Gate-SFT&63.0&25.0&0.055&0.221&91.0\\
\VISTAFocusRow
\VISTAFirstFocus{\method{}}&75.0&36.0&0.120&0.484&\VISTALastFocus 92.0\\
\bottomrule
\end{tabular}
\end{table}

The reported paired measurements combine sensitivity to meaningful changes with stability under irrelevant rewrites. VISTA's target-probability and log-odds differences are $0.120$ and $0.484$, compared with $0.070$ and $0.281$ for Generic-CoT-SFT; its valid-intervention label-change rate is $36.0\%$ and irrelevant-rewrite stability is $92.0\%$. The final training configuration does not use a counterfactual loss ($\lambda_{\mathrm{cf}}=0$).

\subsection{Intervention dimensions and diagnostic weighting}
The six intervention types are goal/concern ($g$), goal congruence ($c$), Expectation ($e$), coping/control ($k$), norm/social relevance ($n$), and expression regulation ($r$), each containing 25 pairs. Table~\ref{tab:counterfactualfields} gives the four available type-specific direction-compliance rates; goal/concern and Expectation have no separately reported rates. This field set differs from the six-field mutual-information diagnostic, which includes agency and excludes free-text goal/concern.

\begin{table}[htbp]
\centering\small
\caption{\textbf{Available VISTA intervention-type results.} Each of the six types has 25 pairs; type-specific percentages are available for the four listed types. The remaining two types contribute to the 150-pair aggregate but have no separately reported rates.}
\label{tab:counterfactualfields}
\begin{tabular}{@{}lrr@{}}
\toprule
\VISTAHeaderRow
\VISTAFirstHeader{Intervention type}&\VISTAHeader{Pairs}&\VISTALastHeader{Direction compliance (\%)}\\
\midrule
Goal congruence&25&84.0\\
Expression regulation&25&81.0\\
Coping / control&25&68.0\\
Social norm&25&69.0\\
\bottomrule
\end{tabular}
\end{table}

\paragraph{A common masking diagnostic.}
Fix the original input's predicted class $y_0=\hat y(X)$ and measure the positive probability drop after masking modality $m$:
\begin{equation}
 \delta_m=\max\{0,p(y_0\mid X)-p(y_0\mid X_{\setminus m})\},\qquad
 \alpha_{\mathrm{diag},m}=\frac{\delta_m+0.01}{\sum_{j\in\{V,A,T\}}\delta_j+0.03}.
 \label{eq:maskingdiagnostic}
\end{equation}
Video masking retains the audio track; audio masking retains text; text masking removes both the transcript and dialogue history. Each masked input is evaluated by a fresh forward pass. With all three modalities included, zero drops yield uniform diagnostic weights of $1/3$. For a diagnostic over $M$ available modalities, the smoothing total is $0.01M$.

The target-modality change in this diagnostic is $+0.040$ for VISTA, compared with the reported control range of $+0.015$--$+0.020$. These diagnostic weights are distinct from the model-native $\alpha$ in Table~\ref{tab:nativeweights}: they summarize output changes across fresh masked-input evaluations and can reflect the fusion residual, the arbitration branch, and their interactions. They describe perturbation response without assigning the full prediction's causal contribution to an individual branch.
\FloatBarrier
\Needspace{10\baselineskip}
\section{Appraisal representation and downstream prediction}
\label{app:learnabilityuse}\label{app:theradia}
\subsection{A shared linear probe on frozen representations}
THERADIA supplies human appraisal annotations for four dimensions. Every core model freezes its backbone and trains the same $\mathrm{Linear}(3584,4)+\mathrm{sigmoid}$ probe. The probe has \mbox{$3584\times4+4=14{,}340$} parameters, so probe capacity is fixed across the five models. The evaluation asks how well those appraisal judgments can be read from each representation, using the concordance correlation coefficient (CCC).

The four targets are novelty, intrinsic pleasantness, goal conduciveness, and coping. Features are standardized with training-set statistics alone. Probe training minimizes mean squared error with learning rate $5\times10^{-4}$, batch size 64, at most 50 epochs, and early-stopping patience 5; the checkpoint is selected by development-set macro CCC. The backbone remains frozen throughout this probe fit.

\begin{table}[htbp]
\centering\small\setlength{\tabcolsep}{4.5pt}
\caption{\textbf{THERADIA appraisal readout: CCC by dimension.} All models use the same four-output linear probe on a frozen backbone. Macro CCC averages the four dimensions.}
\label{tab:theradiafields}
\begin{tabular}{@{}lrrrrr@{}}
\toprule
\VISTAHeaderRow
\VISTAFirstHeader{Method}&\VISTAHeader{Novelty}&\VISTAHeader{Pleasantness}&\VISTAHeader{Goal}&\VISTAHeader{Coping}&\VISTALastHeader{Macro CCC}\\
\midrule
Base&0.400&0.470&0.490&0.520&0.470\\
Emotion-SFT&0.440&0.500&0.530&0.550&0.505\\
Generic-CoT-SFT&0.470&0.530&0.560&0.570&0.533\\
Modality-Gate-SFT&0.470&0.530&0.560&0.580&0.535\\
\VISTAFocusRow
\VISTAFirstFocus{\method{}}&0.540&0.590&0.620&0.650&\VISTALastFocus 0.600\\
\bottomrule
\end{tabular}
\end{table}

VISTA raises macro CCC by $0.095$ over Emotion-SFT and improves all four dimensions. Its aggregate MAE, RMSE, Pearson correlation, and Spearman correlation are $0.086$, $0.111$, $0.658$, and $0.625$, respectively (Table~\ref{tab:theradiaaggregate}). Novelty and Expectation remain different constructs: the four-dimensional probe evaluates the available human dimensions, not complete seven-field correctness.

\begin{table}[htbp]
\centering\small
\caption{\textbf{Aggregate VISTA appraisal-probe metrics on THERADIA.} All quantities retain their native numerical scale.}
\label{tab:theradiaaggregate}
\begin{tabular}{@{}rrrrr@{}}
\toprule
\VISTAHeaderRow
\VISTAFirstHeader{Macro CCC $\uparrow$}&\VISTAHeader{MAE $\downarrow$}&\VISTAHeader{RMSE $\downarrow$}&\VISTAHeader{Pearson $\uparrow$}&\VISTALastHeader{Spearman $\uparrow$}\\
\midrule
0.600&0.086&0.111&0.658&0.625\\
\bottomrule
\end{tabular}
\end{table}

\begin{figure}[htbp]
\centering
\includegraphics[width=\linewidth]{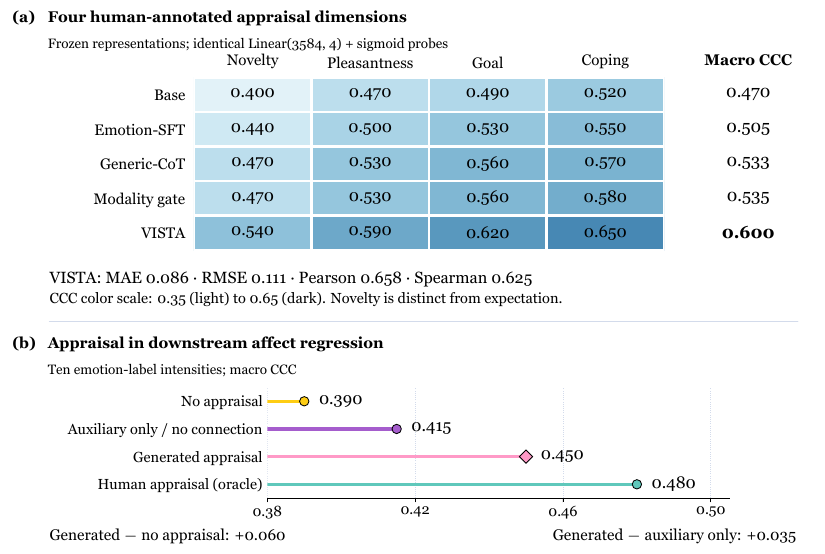}
\caption{\textbf{From readable appraisal to useful appraisal.} (a) Five frozen backbone representations under a common four-dimensional linear probe, evaluated by dimension-wise and macro CCC. (b) Downstream regression of ten emotion intensities under four appraisal conditions. Model-generated appraisal gives $0.450$ macro CCC, versus $0.415$ with auxiliary supervision alone; human appraisal is a privileged test-time reference.}
\label{fig:appraisal_quality}
\end{figure}

\subsection{Using appraisal to predict ten emotion intensities}
The downstream task predicts intensities for ten named emotions: annoyed, anxious, confident, desperate, frustrated, happy, interested, relaxed, satisfied, and surprised. These are the THERADIA core emotion labels, rather than general core-affect axes. Auxiliary appraisal supervision alone gives $0.415$ macro CCC, compared with $0.390$ without appraisal. Connecting model-generated appraisal to the decision raises CCC to $0.450$, a further $0.035$ improvement over auxiliary-only training and $0.060$ over no appraisal. Human appraisal supplies privileged information and reaches $0.480$ (Table~\ref{tab:theradiaemotion}). The four-output probe and ten-output downstream task have distinct targets; the probe's architecture and fitting schedule describe the appraisal-readout experiment.

\begin{table}[htbp]
\centering\small
\caption{\textbf{THERADIA downstream emotion-intensity regression.} The task predicts ten named emotion intensities; the metric is macro CCC. Human appraisal is an oracle condition, separate from inference using model-generated appraisal.}
\label{tab:theradiaemotion}
\begin{tabular}{@{}lrr@{}}
\toprule
\VISTAHeaderRow
\VISTAFirstHeader{Appraisal condition}&\VISTAHeader{Macro CCC}&\VISTALastHeader{Gain over none}\\
\midrule
No appraisal&0.390&0.000\\
Auxiliary only; decision connection cut&0.415&0.025\\
\VISTAFocusRow
\VISTAFirstFocus{Model-generated appraisal}&0.450&\VISTALastFocus 0.060\\
Human appraisal (oracle)&0.480&0.090\\
\bottomrule
\end{tabular}
\end{table}

The probe and downstream experiment serve different purposes. The former measures the readability of human appraisal dimensions at matched probe capacity; the latter tests the usefulness of an appraisal input to affect prediction. Together they support a representation that is both aligned with the available appraisal labels and useful when connected to the downstream predictor. The remaining $0.030$ CCC gap quantifies the distance from generated appraisal to the privileged human-appraisal condition.
\FloatBarrier
\section{Supervision construction, field quality, and label visibility}
\label{app:supervisionreading}\label{app:annotation}
\subsection{Teacher generation and field masks}
One Qwen2.5-Omni-7B teacher generates three samples per input with temperature $0.7$, top-$p$ $0.9$, a maximum of 512 output tokens, and seeds 101, 102, and 103. These repeated samples measure sampling consistency within the same teacher. The fixed template has SHA256
\begin{center}
\footnotesize\ttfamily
c2812f36f754d70875bd69c25534eab6\\
97d57c9f88c34d486b4d6d74c432caa3
\end{center}
Each field stores its value, confidence, evidence, and validity mask separately. Missing values remain null; confidence does not automatically become a supervision weight. Agency value \texttt{unknown} can appear in an output but is masked for supervision. The four binary social-relevance components---politeness, identity, status, and relationship---each have their own mask. A component receives \texttt{false} when evidence supports a negative judgment; absent evidence leaves it missing. This representation supports partial supervision at the level of individual fields and social-relevance components.

\subsection{From candidate pool to retained training set}
The candidate pool contains 10,256 examples. The final common training set retains 8,000: 2,123 from CH-SIMS v2.0, 866 from THERADIA, and 5,011 from MELD. The retention process yields 2,872 automatically accepted examples, 1,231 accepted after human adjudication, and 3,897 retained with partial-field masking; 2,256 candidates are rejected (Table~\ref{tab:retention}). These counts describe the complete candidate disposition. The separate 400-example field audit below is not the count of all human-adjudicated examples.

\begin{table}[htbp]
\centering\small
\caption{\textbf{Complete pseudo-appraisal data accounting.} The left block partitions all candidates; the right block partitions the retained training set by source. Partial masking retains an example while masking unsupported fields.}
\label{tab:retention}
\begin{tabular}{@{}lrlr@{}}
\toprule
\VISTAHeaderRow
\VISTAFirstHeader{Candidate disposition}&\VISTAHeader{Count}&\VISTAHeader{Retained source}&\VISTALastHeader{Count}\\
\midrule
Automatically accepted&2,872&CH-SIMS v2.0&2,123\\
Accepted after human adjudication&1,231&THERADIA&866\\
Retained with partial-field masking&3,897&MELD&5,011\\
Rejected&2,256&&\\
\midrule
All candidates&10,256&All retained&8,000\\
\bottomrule
\end{tabular}
\end{table}

\subsection{Seven-field audit on 400 candidates}
The audit samples 106 CH-SIMS, 43 THERADIA, and 251 MELD candidates, covering conflict, ambiguity, and consistency strata within each source. Two raters independently assess the candidates, with method names and teacher conditions blinded; a third rater adjudicates disagreements. Human agreement refers to the two original judgments before adjudication, whereas teacher sampling consistency concerns the three generated samples.

Coverage uses all 400 audited candidates as its denominator: for example, goal congruence is valid for 296 examples, giving $296/400=74.0\%$. Fields can coexist within a candidate, so their valid counts are not additive. Sampling consistency uses the candidates with sufficiently many valid samples for that field as its denominator. For goal text $g$, semantic comparison uses a frozen goal/concern categorization. Table~\ref{tab:annotation} summarizes quality at the field level. The social-norm row records 136 available fields out of 400 candidates ($34.0\%$ coverage); each of its four binary components retains its own validity mask, so this aggregate is distinct from component-specific coverage.

Grounding judgments have three outcomes: \emph{yes} for explicit input support, \emph{no} for unsupported or contradicted content, and \emph{uncertain} when the evidence does not settle the judgment. The grounded proportion is the number of \emph{yes} judgments divided by all non-missing generated fields; \emph{uncertain} remains in this denominator and is a separate judgment category. Agency and goal congruence have grounded proportions of $84.0\%$ and $81.0\%$; expectation, social norm, and expression regulation have coverage of $43.0\%$, $34.0\%$, and $43.0\%$. The schema therefore combines broadly supported event judgments with more selectively available contextual ones.

\begin{table}[htbp]
\centering\small\setlength{\tabcolsep}{5pt}
\caption{\textbf{Seven-field audit.} Coverage is the valid count divided by 400 candidates. Sampling consistency and grounded proportion are percentages with their field-specific denominators. Multiple fields may be valid in the same example; $n$ retains the reported field-level aggregate.}
\label{tab:annotation}
\begin{tabularx}{\linewidth}{@{}Yrrrr@{}}
\toprule
\VISTAHeaderRow
\VISTAFirstHeader{Field}&\VISTAHeader{Valid $N$}&\VISTAHeader{Coverage}&\VISTAHeader{\shortstack{Sampling\\consistency}}&\VISTALastHeader{Grounded}\\
\midrule
Goal / concern $g$&240&60.0&80.0&76.0\\
Goal congruence $c$&296&74.0&85.0&81.0\\
Expectation $e$&172&43.0&76.0&67.0\\
Agency $a$&260&65.0&86.0&84.0\\
Coping / control $k$&192&48.0&81.0&73.0\\
Social norm $n$&136&34.0&78.0&68.0\\
Expression regulation $r$&172&43.0&80.0&72.0\\
\bottomrule
\end{tabularx}
\end{table}

Coverage and utility are distinct. Expression regulation is valid in $43.0\%$ of audited candidates, while removing the field reduces conflict accuracy by $0.9$ pp in the mechanism comparison. Partial-field masking preserves examples without forcing every field to be inferred from insufficient evidence.

\begin{figure}[htbp]
\centering
\includegraphics[width=\linewidth]{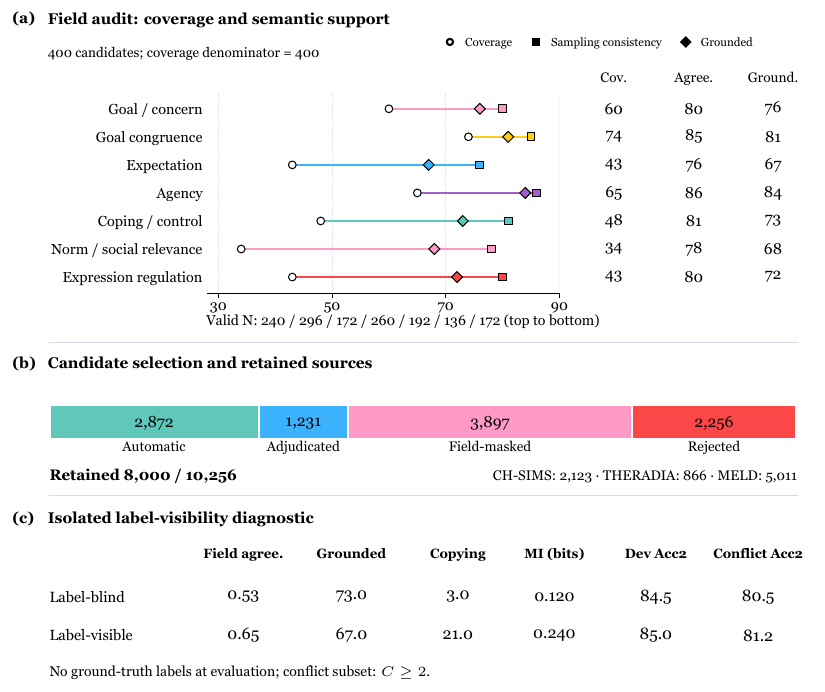}
\caption{\textbf{Supervision from selection to decision quality.} (a) Coverage, sampling consistency, grounded proportion, and valid count for all seven fields in the 400-candidate audit. (b) Disposition of all 10,256 candidates and dataset sources of the 8,000 retained examples. (c) Label-blind versus label-visible teacher diagnostics, with the latter trained in an isolated copy. Percentages, field-consistency scores, and mutual information retain their distinct units.}
\label{fig:supervision}
\end{figure}

\subsection{An isolated label-visibility diagnostic}
The main pipeline generates appraisal without revealing the true emotion label to the teacher. A label-visible comparison uses an isolated training copy. Both are evaluated on CH-SIMS development examples without supplying true labels at evaluation, and the conflict subset uses $C\geq2$.

Label-copy rate counts field outputs containing a predefined target-label term or synonym, divided by all auditable outputs; quotations from the input evidence are recorded separately. Grounding, copying, and sampling consistency describe complementary properties of generated supervision. Human agreement is computed from the raters' original decisions and is distinct from the field-consistency score reported in Table~\ref{tab:labelvisibility}.

The mutual-information diagnostic uses $\mathcal F=\{c,e,a,k,n,r\}$ and the CH-SIMS binary sentiment label $Y$. Goal text $g$ is excluded. Coping $k$ is discretized into five fixed equal-width bins, $n$ is encoded as a four-bit combination, and missing states are represented separately; agency \texttt{unknown} is not a valid attribution category. Each field--label contingency table uses Dirichlet $0.5$ smoothing, with information measured in bits. For each field, 1,000 label permutations estimate the finite-sample bias:
\begin{equation}
 I_f^{\mathrm{corr}}=\max\!\left\{0,\;I(Z_f;Y)-\frac{1}{1000}\sum_{b=1}^{1000}I(Z_f;Y^{(b)})\right\},
 \qquad I_{\mathrm{app}}=\frac{1}{6}\sum_{f\in\mathcal F}I_f^{\mathrm{corr}}.
 \label{eq:appraisalmi}
\end{equation}
Thus the reported MI is an equal-weight mean of six bias-corrected field scores.

\begin{table}[htbp]
\centering\small\setlength{\tabcolsep}{3.8pt}
\caption{\textbf{Label-visibility diagnostic.} Field consistency retains the reported score scale; grounded proportion, copying, and Acc2 are percentages. MI is the mean bias-corrected field--label mutual information in bits (Equation~\ref{eq:appraisalmi}). Both development evaluations hide true labels.}
\label{tab:labelvisibility}
\begin{tabular}{@{}lrrrrrr@{}}
\toprule
\VISTAHeaderRow
\VISTAFirstHeader{Teacher condition}&\VISTAHeader{Consistency}&\VISTAHeader{Grounded}&\VISTAHeader{Copying}&\VISTAHeader{MI (bits)}&\VISTAHeader{Dev Acc2}&\VISTALastHeader{Conflict Acc2}\\
\midrule
Label-blind&0.53&73.0&3.0&0.120&84.5&80.5\\
Label-visible, isolated&0.65&67.0&21.0&0.240&85.0&81.2\\
\bottomrule
\end{tabular}
\end{table}

Making the teacher label-aware increases field consistency from $0.53$ to $0.65$ and appraisal--label MI from $0.120$ to $0.240$ bits, while copying rises from $3.0\%$ to $21.0\%$ and grounding falls from $73.0\%$ to $67.0\%$. Development accuracy increases by $0.5$ pp overall and $0.7$ pp on the conflict subset. Higher label association and higher scene grounding are consequently different properties. The observed MI increase alone neither establishes source leakage nor replaces a grounding assessment; the isolated label-visible condition is a diagnostic of supervision design.

\subsection{Second-pass explanation audit}
A separate pilot audit uses 150 examples, with 50 examples in each of its three strata. The first prediction is locked before a second-pass explanation is generated. The grounded-appraisal rates are $52.0\%$ for Direct and $60.0\%$ for Appraisal prompt (Figure~\ref{fig:score_geometry}). This protocol measures explanation quality after the recognition decision; the second pass does not update that decision.
\FloatBarrier
\section{Computation and the structure of remaining errors}
\label{app:costerrorreading}\label{app:cost}
\subsection{Training and inference resources}
Table~\ref{tab:costfull} compares training and inference resources for the four methods. The four methods share 8,000 training examples, three epochs, 375 optimization steps, effective batch size 64, LoRA rank 16, and seeds 42, 43, and 44. Their auxiliary heads yield different trainable parameter counts. Generic-CoT-SFT and VISTA share a mean target length of 360 tokens, allowing computation to be compared at matched average output length.

\begin{table}[htbp]
\centering\small\setlength{\tabcolsep}{3.8pt}
\caption{\textbf{Complete resource accounting.} Training columns concern the common stage; inference is measured on one A100 at batch size 1. GPU-hours are per seed and peak memory is GiB per GPU. P50 and P95 are inference latency percentiles in seconds.}
\label{tab:costfull}
\begin{tabular}{@{}lrrrrrr@{}}
\toprule
\VISTAHeaderRow
\VISTAFirstHeader{Method}&\VISTAHeader{Params (M)}&\VISTAHeader{GPU-h}&\VISTAHeader{Peak GiB}&\VISTAHeader{P50 (s)}&\VISTAHeader{P95 (s)}&\VISTALastHeader{Tokens}\\
\midrule
Emotion-SFT&10.12&32.0&42.0&1.2&4.0&8\\
Generic-CoT-SFT&10.12&40.0&45.0&13.5&26.0&360\\
Modality-Gate-SFT&12.19&64.0&52.0&6.0&13.0&96\\
\VISTAFocusRow
\VISTAFirstFocus{\method{}}&12.39&72.0&56.0&15.5&31.0&\VISTALastFocus 360\\
\bottomrule
\end{tabular}
\end{table}

VISTA uses $72$ GPU-h per seed versus $40$ for Generic-CoT-SFT, and has P50 latency $15.5$ s versus $13.5$ s. Its three-seed common training totals $216$ GPU-h. That total excludes teacher generation, separate task adaptation, ablations, and extra diagnostic calls. The resources measure the cost of the evaluated systems, rather than claiming identical auxiliary capacity or total computation across methods.

\begin{figure}[htbp]
\centering
\includegraphics[width=\linewidth]{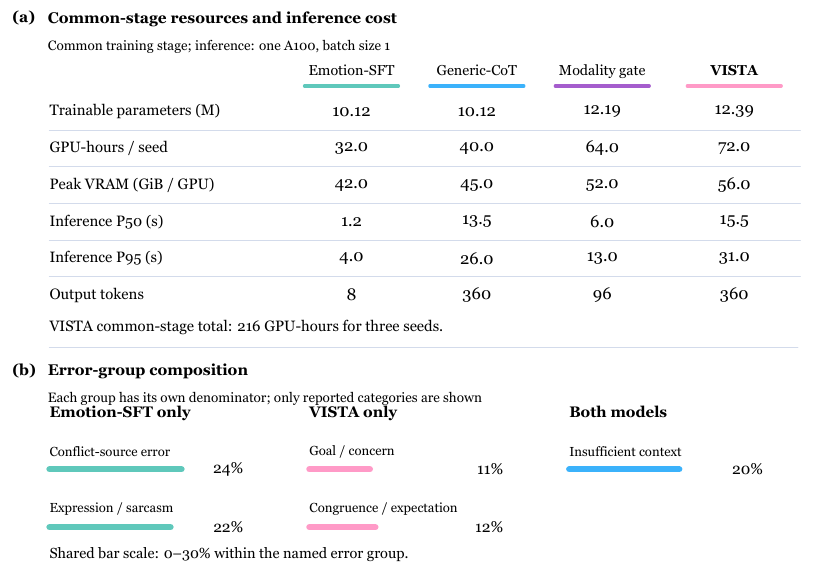}
\caption{\textbf{Resources and error groups.} (a) All six reported training and inference resource measurements for the four trained methods. (b) The five available error-category shares, separated by the three paired-model error groups that define their denominators. These shares are not a complete error partition and are not pooled across groups.}
\label{fig:cost_errors}
\end{figure}

\subsection{Errors corrected, introduced, and shared}
The paired error analysis groups examples by whether Emotion-SFT, VISTA, or both make an error. Within examples where only Emotion-SFT errs, conflict-source mistakes account for $24.0\%$ and expression-regulation/sarcasm mistakes for $22.0\%$. Within examples where only VISTA errs, goal/concern mistakes account for $11.0\%$ and congruence/Expectation mistakes for $12.0\%$. Among examples where both models err, insufficient context accounts for $20.0\%$ (Table~\ref{tab:errors}).

\begin{table}[htbp]
\centering\small
\caption{\textbf{Available paired-error audit results.} Each percentage uses the group in its first column as the denominator. The table lists the reported categories rather than a full partition; entries must not be summed across error groups or interpreted as corpus prevalence.}
\label{tab:errors}
\begin{tabularx}{\linewidth}{@{}>{\raggedright\arraybackslash}p{3.8cm}Yr@{}}
\toprule
\VISTAHeaderRow
\VISTAFirstHeader{Error group}&\VISTAHeader{Reported error category}&\VISTALastHeader{Share (\%)}\\
\midrule
Only Emotion-SFT wrong&Conflict-source identification&24.0\\
Only Emotion-SFT wrong&Expression regulation / sarcasm&22.0\\
Only VISTA wrong&Goal / concern&11.0\\
Only VISTA wrong&Goal congruence / Expectation&12.0\\
Both models wrong&Insufficient context&20.0\\
\bottomrule
\end{tabularx}
\end{table}

This grouping links the method's gain to concrete interpretation problems without obscuring errors introduced by the intermediate appraisal. It also separates an absent contextual cue from a failure to interpret an available one: better arbitration can improve the latter, while the former requires information beyond the observed input.
\FloatBarrier
\section{Ordinary conversational recognition and class recall}
\label{app:ordinary}
MELD tests whether the conflict-oriented representation remains useful for ordinary seven-class conversational emotion recognition. Trained models receive the task-specific adaptation described in Appendix~\ref{app:pilot}; Base is evaluated zero-shot. VISTA reaches $66.94\%$ weighted-F1, $52.89\%$ macro-F1, and $67.85\%$ accuracy (Table~\ref{tab:meld}). Its reported weighted-F1 gain over Emotion-SFT is $1.46$ pp, computed from the unrounded scores.

\begin{table}[htbp]
\centering\small
\caption{\textbf{Complete MELD results (\%).} Differences are in pp relative to Emotion-SFT and are calculated before rounding, so they need not equal subtraction of the displayed rounded means.}
\label{tab:meld}
\begin{tabular}{@{}lrrrr@{}}
\toprule
\VISTAHeaderRow
\VISTAFirstHeader{Method}&\VISTAHeader{Weighted-F1}&\VISTAHeader{Macro-F1}&\VISTAHeader{Accuracy}&\VISTALastHeader{$\Delta$wF1}\\
\midrule
Base&59.55&43.73&60.59&$-5.93$\\
Emotion-SFT&65.49&50.65&66.48&0.00\\
Generic-CoT-SFT&65.69&51.12&66.64&$+0.20$\\
Modality-Gate-SFT&65.74&50.84&66.73&$+0.25$\\
\VISTAFocusRow
\VISTAFirstFocus{\method{}}&66.94&52.89&67.85&\VISTALastFocus $+1.46$\\
\bottomrule
\end{tabular}
\end{table}

The final evaluation reports increased recall in all seven emotion classes relative to Emotion-SFT. Table~\ref{tab:meldrecall} gives the available class-specific measurements: disgust rises from $18.0$ to $21.0\%$, fear from $20.0$ to $23.0\%$, and sadness from $39.0$ to $41.5\%$. VISTA's neutral recall is $81.7\%$. These improvements accompany substantial remaining differences among emotion classes, which weighted-F1 alone does not display.

\begin{table}[htbp]
\centering\small
\caption{\textbf{Available MELD class-recall measurements (\%).} The table includes the available numerical recall values; a dash marks an unreported comparator value.}
\label{tab:meldrecall}
\begin{tabular}{@{}lrr@{}}
\toprule
\VISTAHeaderRow
\VISTAFirstHeader{Emotion}&\VISTAHeader{Emotion-SFT}&\VISTALastHeader{VISTA}\\
\midrule
Disgust&18.0&21.0\\
Fear&20.0&23.0\\
Sadness&39.0&41.5\\
Neutral&---&81.7\\
\bottomrule
\end{tabular}
\end{table}

The ordinary-recognition result complements the conflict gains rather than imposing a separate non-inferiority test. The final aggregate report does not supply uncertainty intervals, and the paper makes no formal significance or non-inferiority determination from the displayed differences.
\FloatBarrier
\section{Shared training and task adaptation}
\label{app:pilot}\label{app:implementation}
\subsection{What is held fixed across the core comparison}
All five core models use Qwen2.5-Omni-7B. Base is the frozen pretrained model. Emotion-SFT learns the final emotion or sentiment label; Generic-CoT-SFT adds generic reasoning; Modality-Gate-SFT learns modality gating; VISTA combines structured appraisal, conflict recognition, and semantic arbitration. The four trained methods share the common-stage conditions in Table~\ref{tab:sharedtraining}. Auxiliary heads differ, as reflected in the parameter counts in Table~\ref{tab:costfull}.

\begin{table}[htbp]
\centering\small
\caption{\textbf{Common-stage training configuration.} These settings apply to all four trained core methods. Additional task adaptation and the common THERADIA probe are specified separately.}
\label{tab:sharedtraining}\label{tab:pilot}
\begin{tabularx}{\linewidth}{@{}p{4.1cm}Y@{}}
\toprule
\VISTAHeaderRow
\VISTAFirstHeader{Setting}&\VISTALastHeader{Final configuration}\\
\midrule
Backbone&Qwen2.5-Omni-7B Thinker\\
Frozen components&Audio encoder, visual encoder, and Talker\\
Retained training examples&8,000\\
Epochs / optimization steps&3 / 375\\
Effective batch size&64\\
LoRA targets&\texttt{q\_proj}, \texttt{k\_proj}, \texttt{v\_proj}, \texttt{o\_proj} in 28 layers\\
LoRA rank / alpha / dropout&16 / 32 / 0.05\\
Optimizer&AdamW; $\beta=(0.9,0.95)$; weight decay $0.01$\\
Learning rates&LoRA: $5\times10^{-5}$; added heads: $10^{-4}$\\
Schedule&Cosine; warmup for $5\%$ of optimizer steps\\
Numerical configuration&bf16; gradient-norm clipping at $1$\\
Input budget&8,192 total tokens; 1,024 text tokens; at most 16 video frames\\
Decoding&Greedy; maximum 512 generated tokens, apart from Base's short classification interface\\
Random seeds&42, 43, 44\\
Generic-CoT / VISTA mean target length&360 / 360 tokens\\
VISTA counterfactual coefficient&$\lambda_{\mathrm{cf}}=0$\\
\bottomrule
\end{tabularx}
\end{table}

The pretrained weight revision is \texttt{ae9e1690543ffd5c0221dc27f79834d0294cba00}. The common-stage step count is $8{,}000\times3/64=375$; downstream task adaptation is a separate optimization stage. Checkpoint selection maximizes the equally weighted mean of MELD development accuracy and CH-SIMS development Acc2 after putting both on the same numerical scale; a tie selects the earlier checkpoint. The development criteria are separate from the held-out test measurements.

\subsection{Training objective and missing supervision}
The final training objective combines the emotion task, unimodal tasks, appraisal fields, conflict recognition, and arbitration:
\begin{equation}
 \mathcal L=\mathcal L_{\mathrm{emo}}+1.0\mathcal L_{\mathrm{uni}}
 +0.5\mathcal L_{\mathrm{app}}+0.5\mathcal L_{\mathrm{conf}}
 +0.5\mathcal L_{\mathrm{arb}},\qquad\lambda_{\mathrm{cf}}=0.
 \label{eq:trainingobjective}
\end{equation}
The ablation notation uses the aliases $\mathcal L_z=\mathcal L_{\mathrm{app}}$, $\mathcal L_c=\mathcal L_{\mathrm{conf}}$, and $\mathcal L_a=\mathcal L_{\mathrm{arb}}$. Classification targets use cross-entropy and continuous targets use mean squared error; a unimodal task contributes only when that modality has a valid target. Every supervised component is normalized over its valid targets. For component $j$,
\begin{equation}
 \mathcal L_j=\frac{\sum_i w_i M_{ij}\ell_{ij}}
 {\sum_i w_i M_{ij}+\epsilon},\qquad
 w_i=\begin{cases}1+0.5C_i,&\text{a valid conflict label is available},\\
 1,&\text{otherwise},\end{cases}
 \label{eq:maskedtrainingloss}
\end{equation}
where $M_{ij}$ is the supervision-validity mask and $\epsilon$ stabilizes the denominator. The conflict quantity $C_i$ is defined in Equation~\ref{eq:simsconflictintensity}. A missing target contributes neither a loss nor a denominator weight.

Within appraisal, $c,e,a,r$ use categorical losses, $k$ uses regression, and the four $n$ indicators use individually masked binary cross-entropy. The goal/concern field $g$ uses a text loss divided by its number of valid target tokens. Each field is first normalized internally; the sample-level appraisal loss then averages its valid fields. A sample with no valid appraisal fields is excluded from the appraisal denominator. Thus target length and the number of social-relevance bits do not implicitly determine field weight. Missing values remain null, agency \texttt{unknown} is excluded as a supervision target, and an $n$ bit is false only when negative evidence supports it. Confidence and validity masks are separate quantities; confidence is not automatically used as a loss weight.

\subsection{Checkpoint use across evaluation tasks}
Table~\ref{tab:adaptation} distinguishes direct evaluation, task adaptation, and probing. EmoMM evaluates the common-stage checkpoint without additional adaptation; CH-SIMS and MELD use adapted task predictors, while THERADIA compares frozen representations at equal probe capacity. These checkpoint roles describe optimization stages; source-corpus relationships are specified in Appendix~\ref{app:data}.

\begin{table}[htbp]
\centering\small
\caption{\textbf{Evaluation-stage use of the core models.} The labels identify distinct experimental protocols rather than a single checkpoint condition for every task.}
\label{tab:adaptation}
\begin{tabularx}{\linewidth}{@{}>{\raggedright\arraybackslash}p{2.1cm}YY@{}}
\toprule
\VISTAHeaderRow
\VISTAFirstHeader{Evaluation}&\VISTAHeader{Base}&\VISTALastHeader{Four trained methods}\\
\midrule
CA-MER&Frozen Base&Common-stage trained models; normalized output interface\\
EmoMM&Frozen Base&Common-stage checkpoint; no EmoMM adaptation\\
CH-SIMS v2.0&Zero-shot&Separate task adaptation; train-defined conflict groups\\
MELD&Zero-shot&Separate task adaptation for seven emotion classes\\
THERADIA probe&Frozen backbone + common probe&Frozen backbone + common probe\\
\bottomrule
\end{tabularx}
\end{table}

\subsection{The complete cross-task result overview}
Table~\ref{tab:alltaskoverview} places the six reported headline measurements together. The comparisons retain each task's metric and evaluation protocol. They are not averaged into an artificial universal score: conflict accuracy, appraisal CCC, and ordinary weighted-F1 answer different parts of the scientific question.

\begin{table}[htbp]
\centering\small\setlength{\tabcolsep}{3.8pt}
\caption{\textbf{Complete six-metric overview.} Classification and F1 values are percentages; THERADIA CCC retains its native scale. CA-MER conflict is the macro average of its two directions, CH-SIMS Q4 is binary accuracy, and THERADIA is the four-dimensional probe macro CCC.}
\label{tab:alltaskoverview}
\begin{tabular}{@{}lrrrrrr@{}}
\toprule
\VISTAHeaderRow
\VISTAFirstHeader{Method}&\VISTAHeader{CA conflict}&\VISTAHeader{CA consistent}&\VISTAHeader{EmoMM conflict}&\VISTAHeader{SIMS Q4}&\VISTAHeader{CCC}&\VISTALastHeader{MELD wF1}\\
\midrule
Base&54.5&68.0&46.5&70.0&0.470&59.55\\
Emotion-SFT&60.0&73.0&47.5&77.0&0.505&65.49\\
Generic-CoT-SFT&61.5&74.0&48.0&77.8&0.533&65.69\\
Modality-Gate-SFT&62.0&74.0&49.0&79.2&0.535&65.74\\
\VISTAFocusRow
\VISTAFirstFocus{\method{}}&64.5&74.2&52.0&81.5&0.600&\VISTALastFocus 66.94\\
\bottomrule
\end{tabular}
\end{table}
\FloatBarrier
\section{Reporting conventions and reproducibility map}
\label{app:repro}
\subsection{Quantities, denominators, and precision}
Appraisal $z$ is an event-level state containing goal/concern, goal congruence, Expectation, agency, coping/control, social norm, and expression regulation. Classification accuracy, F1, recall, and other rates are percentages; their differences are percentage points. CCC, MAE, RMSE, correlations, probability differences, and modality weights retain their native scales. Appendix~\ref{app:pilot} specifies training and adaptation.

For the recognition and readout comparisons, the reporting rule computes each metric from the unrounded predictions of seeds 42, 43, and 44 separately, averages the three seed-level metrics equally, and rounds once for display. In particular, F1 and CCC are computed within each seed before averaging. Teacher sampling, shuffling, and bootstrap randomness do not create additional independent training runs. Reported contrasts use the unrounded evaluation values: MELD's VISTA--Emotion-SFT weighted-F1 gain is $1.46$ pp, versus $1.45$ pp from the rounded means. Newly derived contrasts (Table~\ref{tab:quartilecontrasts}) use displayed values. The paired analyses retain the aggregate precision specified in Appendix~\ref{app:counterfactualreading}. Aggregate scores alone do not determine seed standard deviations, confidence intervals, significance tests, or non-inferiority decisions.

Denominators are 10,256 candidates for selection; 8,000 retained examples for source counts; 400 candidates for coverage; 150 pairs each for valid interventions and irrelevant rewrites; and the corresponding error group for each error share. The six intervention types each have 25 pairs; aggregate rates are not inverted into per-run counts. For direction compliance, invalid model outputs are failures in the fixed evaluation denominator. Continuous probability and log-odds summaries require defined probabilities and use their own valid counts. Native $\alpha$ and masking-based $\alpha_{\mathrm{diag}}$ remain distinct.

\paragraph{Training history and source identity.}
Common training, task adaptation, and development-set checkpoint selection define different uses of a sample. Source identity spans those stages and can differ from a dataset's local clip ID: two differently named examples may share a video, dialogue, session, subject, or overlapping media interval. A source-disjointness claim therefore requires a canonical source comparison as well as exact clip/media matching. No EmoMM-specific adaptation, as used here, describes the optimization protocol; it does not establish source disjointness from the shared CH-SIMS v2 training material.

\subsection{A complete map of the final experimental evidence}
Table~\ref{tab:evidencemap} locates all experiments, including the output-validity, prompt, external-method, class-recall, and resource analyses that complement the main comparisons.

\begin{table}[htbp]
\centering\small
\caption{\textbf{A guide to the complete experimental evidence.} All experiment families and their numerical results.}
\label{tab:evidencemap}
\fontsize{9}{10}\selectfont
\renewcommand{\arraystretch}{1}
\begin{tabularx}{\linewidth}{@{}>{\raggedright\arraybackslash}p{3.2cm}Y>{\raggedright\arraybackslash}p{2.8cm}@{}}
\toprule
\VISTAHeaderRow
\VISTAFirstHeader{Evidence family}&\VISTAHeader{Question and complete contents}&\VISTALastHeader{Location}\\
\midrule
Core settings and definitions&Backbone, shared training, task adaptation, units, and six-metric overview&Tables~\ref{tab:sharedtraining}--\ref{tab:alltaskoverview}; this section\\
CA-MER recognition&All subsets, overall accuracy, invalid output, external reevaluations&Tables~\ref{tab:camercomplete}, \ref{tab:externalcamer}\\
Conflict specificity&All three control contrasts and directional decomposition&Tables~\ref{tab:specificityfull}, \ref{tab:directionalcontrasts}\\
EmoMM&All four conditions and CHASE reevaluation&Table~\ref{tab:emomm}\\
CH-SIMS v2.0&All conflict groups, overall Acc2, Q4 MAE&Table~\ref{tab:chsims}\\
MELD&Weighted-F1, macro-F1, accuracy, unrounded differences, class recall&Tables~\ref{tab:meld}, \ref{tab:meldrecall}\\
THERADIA&All five four-dimensional probes, aggregate metrics, downstream conditions&Tables~\ref{tab:theradiafields}--\ref{tab:theradiaemotion}\\
Mechanism controls&All 13 appraisal/training configurations and native weights&Tables~\ref{tab:fullablations}, \ref{tab:nativeweights}\\
Counterfactual behavior&All five paired metrics, available type rates, masking diagnostic&Tables~\ref{tab:counterfactual}, \ref{tab:counterfactualfields}; Appendix~\ref{app:counterfactualreading}\\
Frozen prompting&All four prompts, length, invalid output, second-pass grounding audit&Table~\ref{tab:frozenmain}; Figure~\ref{fig:score_geometry}\\
Supervision&Candidate/source counts, all seven fields, six label-visibility metrics&Tables~\ref{tab:retention}--\ref{tab:labelvisibility}\\
Resources and errors&All six cost metrics and five reported group-specific error shares&Tables~\ref{tab:costfull}, \ref{tab:errors}\\
\bottomrule
\end{tabularx}
\end{table}

\paragraph{Public-label prevalence is an independent analysis.}
The public CH-SIMS audit uses a fixed sentiment mapping on all 2,281 released clips, finding 1,117 conflicts independently of VISTA predictions and training selection (Appendix~\ref{app:conflictdefinition}). Appendix~\ref{app:conflictdefinition} specifies the released data, fixed label mapping, exact comparison rule, and split-level counts.

\paragraph{Scope of each empirical claim.}
The comparisons support conflict gains, appraisal readout and downstream utility, and intervention sensitivity. Prompting, missingness, class recall, and latency characterize their operating conditions. These aggregates do not separately estimate the Bayes-risk terms or fixed-evidence logit contrasts developed in Appendices~\ref{app:theoryextra} and~\ref{app:pathidentification}.
\FloatBarrier
\paragraph{Responsible interpretation of event appraisal.}
\label{app:ethics}
Appraisal is an event-specific hypothesis, not a durable personal attribute or moral judgment. Goals, norms, and expression vary across people and settings; unsupported fields should remain uncertain or masked. Consequential applications require context-specific validation and human review. Conversational and audiovisual data require participant privacy and compliance with dataset access conditions.
\stopcontents[vistaevidence]
\end{document}